\documentclass[11pt,a4paper]{article}
\usepackage{graphicx}
\usepackage{amsmath}
\usepackage{verbatim}
\usepackage{color}
\usepackage{float}
\usepackage{authblk}
\usepackage[colorlinks]{hyperref}
\usepackage{doi}
\usepackage{textcomp}
\usepackage{eurosym}
\usepackage{amssymb}
\usepackage{pifont}
\usepackage{array} 
\usepackage{xcolor}
\emergencystretch=\maxdimen
\usepackage{amsthm,amsfonts,amssymb,amsmath,graphicx}

\usepackage{algorithm}
\usepackage{algorithmic}
\usepackage{url}
\usepackage{hyperref}
\numberwithin{equation}{section}

\hypersetup{
    colorlinks=true,
    linkcolor=green,   
    citecolor=black,   
    urlcolor=black     
}

\usepackage{multirow}
\graphicspath{{Fig/}}
\usepackage{algorithm}
\usepackage{subcaption}

\newcommand{\bm}[1]{\mbox{\boldmath$#1$}}
\def\be{\begin{equation}}
\def\ee{\end{equation}}

\def\bal{\begin{aligned}}
\def\eal{\end{aligned}}
\def\bes{\begin{equation*}}
\def\ees{\end{equation*}}

\newtheorem{theorem}{Theorem}[section]

\theoremstyle{remark}

\usepackage{tikz}
\usepackage{longtable} 
\usepackage{mathabx} 
\usepackage{ragged2e}
\usepackage{array}  
\usepackage{booktabs} 
\usepackage{siunitx} 
\usepackage{multirow} 
\usepackage{rotating}
\usepackage{enumitem}
\usepackage{mathtools}

\usepackage{listings}
\usepackage{hyperref}
\definecolor{codegreen}{rgb}{0,0.6,0}
\definecolor{codegray}{rgb}{0.5,0.5,0.5}
\definecolor{codepurple}{rgb}{0.58,0,0.82}
\definecolor{backcolour}{rgb}{0.95,0.95,0.92}
\usepackage{caption} 

\usepackage{listings}
\usepackage{xcolor}
\lstdefinestyle{mypython}{
    language=Python,
    basicstyle=\ttfamily\small,
    keywordstyle=\color{blue}\bfseries,
    stringstyle=\color{red},
    commentstyle=\color{green!50!black},
    numbers=left,
    numberstyle=\tiny\color{gray},
    stepnumber=1,
    showstringspaces=false,
    breaklines=true,
    frame=single,
    tabsize=4,
    morekeywords={self,Variable,geometry, TimeDomain,
    icbc,IC,grad, jacobian,torch,sigmoid,PointSetBC,
    data,PDE,nn,FNN,Model,compile,train,predict} 
}

\usepackage{geometry} 
\makeatletter
\renewenvironment{abstract}{%
  \if@twocolumn
    \section*{\abstractname}%
  \else
    \small
    \begin{trivlist}%
      \item[\hskip \labelsep \bfseries \abstractname] 
  \fi}
  {\if@twocolumn\else\end{trivlist}\fi}
\makeatother

\begin{document}
\title{\textbf{PBPK-iPINNs: Inverse Physics-Informed Neural Networks for Physiologically Based Pharmacokinetic Brain Models}}
\author[]{Charuka D. Wickramasinghe$^{\dagger}$, Krishanthi C. Weerasinghe$^{\ddagger}$, Pradeep K. Ranaweera$^{\ddagger}$, and Nelum S.S.M.  Hapuhinna$^{\ddagger}$}
\date{}
\maketitle
\begin{abstract} Physics-Informed Neural Networks (PINNs) integrate machine learning with differential equations to solve forward and inverse problems while ensuring that predictions adhere to physical laws. Physiologically based pharmacokinetic (PBPK) modeling advances beyond classical compartmental approaches by employing a mechanistic, physiology-focused framework. Such models involve many unknown parameters that are difficult to measure directly in humans due to ethical and practical constraints. PBPK models are constructed as systems of ordinary differential equations (ODEs) and these parametric ODEs are often stiff, and traditional numerical and statistical methods frequently fail to converge. In this study, we consider a permeability-limited, four-compartment PBPK brain model that mimics human brain functionality in drug delivery. We introduce PBPK-iPINN, a method for estimating drug-specific or patient-specific parameters and drug concentration profiles using inverse PINNs. We also conducted parameter identifiability analysis to determines whether the parameters can be uniquely and reliably estimated from the available data.  We demonstrate that, for the inverse problem to converge to the correct solution, the components of the loss function (data loss, initial condition loss, and residual loss) must be appropriately weighted, and the hyperparameters including the number of layers and neurons, activation functions, learning rate, optimizer, and collocation points must be carefully tuned. The performance of the PBPK-iPINN approach is then compared with established numerical and statistical methods. PBPK-iPINN achieved parameter errors on the order of $10^{-8}$ for volumes, outperforming traditional methods. Accurate parameter estimation yields precise drug concentration-time profiles, which in turn enable the calculation of pharmacokinetic metrics. These metrics support drug developers and clinicians in designing and optimizing therapies for brain cancer.


\end{abstract}

\renewcommand{\thefootnote}{\fnsymbol{footnote}}
\footnotetext[2]{Department of Oncology, Wayne State University, MI (gi6036@wayne.edu).}
\footnotetext[3]{Department of Chemistry, Siena Heights University, MI (kweerasi@sienaheights.edu).}
\footnotetext[3]{Department of Engineering, Siena Heights University, MI (pranawee@sienaheights.edu).}
\footnotetext[3]{Department of Mathematics and Statistics, Northern Kentucky University, KY (hapuhinnan1@nku.edu).}
\renewcommand{\thefootnote}{\arabic{footnote}}

\section{Introduction}  Solving ordinary differential equations (ODEs) and/or partial differential equations (PDEs) analytically is rarely feasible. Even when exact solutions exist, interpreting their behavior can be difficult. To address this, numerical methods based on discretization are commonly employed\cite{atkinson2009numerical,cheney1998numerical,teschl2021ordinary}. Although numerical methods for solving ODEs (Euler’s method, Runge-Kutta methods, and etc) and PDEs (finite element method, finite difference method, and etc)  have made great progress in simulating complex problems, they still face major challenges. These include difficulty in directly using noisy data, the complexity of mesh generation \cite{li2025analysis, wickramasinghe2024graded, wickramasinghe2024numerical,wickramasinghe2022c0}, and the inability to efficiently handle high-dimensional parameterized differential equations. Inverse problems with unknown physics are especially costly to solve, often demanding separate formulations and specialized code. Although parameter estimation for differential equation systems has been extensively studied \cite{liang2008parameter, calver2019parameter, kaschek2012variational}, our focus is on estimating drug-specific and/or system-specific parameters of  physiologically based pharmacokinetic brain compartmental models.

While machine learning offers a promising alternative, training deep neural networks typically requires large datasets which are often unavailable in scientific applications. A more practical approach is to train such networks using additional information derived from enforcing physical laws. The figure (\ref{dp1}) provides a schematic overview of three types of physical problems along with their corresponding available data. Purely data-driven models may fit observations very well, but predictions may be physically inconsistent or implausible, owing to extrapolation or observational biases that may lead to poor generalization performance.
To this end, physics-informed learning is needed. This refers to using prior knowledge whether from observations, experiments, physics, or mathematics to enhance the performance of a learning algorithm.

In recent years, physics-informed neural networks (PINNs) \cite{raissi2019physics, hariri2025physics, karniadakis2021physics} have seen remarkable growth. These methods effectively combine the fundamental physical principles expressed through differential equations. PINNs are well suited for addressing inverse problems, either by serving as surrogate models in combination with standard parameter estimation methods (Bayesian or deterministic) or by directly performing the estimation as independent tools. Inverse problems are inherently more complex than forward problems due to their potentially ill-posed nature, where multiple solutions may exist or no solution at all. Challenges often arise in data-scarce regimes, irregular geometries, missing measurements, or from uncertainties inherent in the model. Advanced PINN techniques have been developed to address these difficulties. For a recent review, see \cite{yang2021b, gusmao2024maximum, difonzo2024physics} and references therein.

\begin{figure}[H]
\centering
\includegraphics[width=10cm]{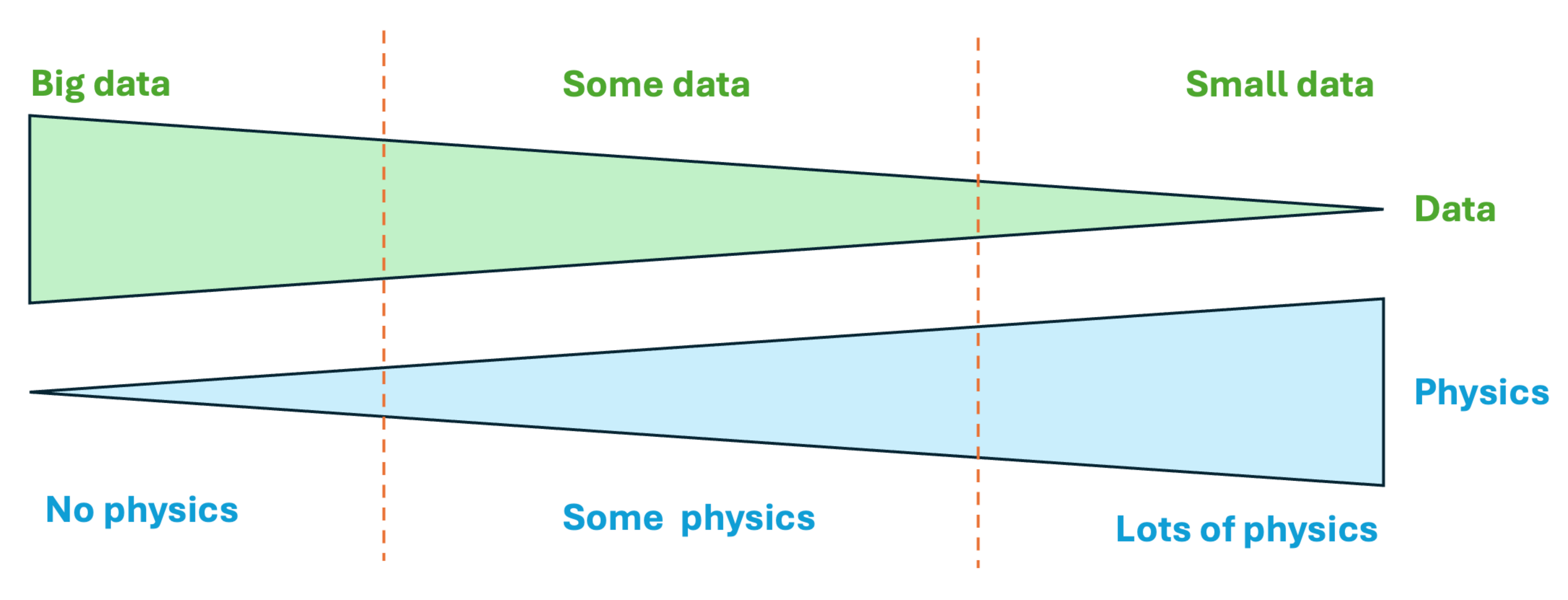}
\begin{tiny}
\caption{Schematically illustrates three possible categories of physical problems and associated available data.}\label{dp1}
\end{tiny}
\end{figure}

Parameter identifiability is a fundamental concept in system modeling that determines whether model parameters can be uniquely estimated from available data. Using the four-compartment brain model, we assess the structural identifiability of the model parameters. When parameters are found to be structurally unidentifiable, we fix selected parameters, leveraging the richness of the available dataset. Subsequently, using the parameters inferred through the PBPK-iPINN framework, we perform a practical identifiability analysis to evaluate the quality and reliability of the parameter estimates

In this work, we present an inverse physics-informed neural network (iPINN) to estimate unknown parameters of PBPK brain compartment models and to obtain the drug concentration profiles in different regions of the human brain which gives the best fit to a set of experimental data. These physiologically based pharmacokinetic models consider various factors, such as the drug's ability to cross the blood-brain barrier, its rate of clearance from the body, and how it interacts with tissues \cite{zhuang2016pbpk, gaohua2016development}. These models consist of numerous system-specific(i.e., human body) and drug specific  parameters. While experimental methods for estimating these parameters are well-established and can provide valuable insights, several challenges remain \cite{Paper1}. These include inter-individual variability, limitations in analytical techniques, complexity of drug metabolism, drug-drug interactions, obtaining sufficient and representative blood and tissue samples, ethical and practical constraints, and environmental factors. As a result, there is a growing interest in leveraging mathematical modeling, statistical inference, and machine learning approaches to estimate these parameters more robustly and efficiently.

Due to the complexity of the human body, drug kinetics are often modeled using one or more interconnected compartments, each representing a group of tissues with similar blood flow. These compartments are conceptual rather than actual anatomical or physiological regions, and the drug is assumed to be uniformly distributed within each \cite{chaudhry2016pharmacokinetic}. Compartment modeling in the field of pharmacokinetics and pharmacodynamics (PKPD), uses a mathematical approach to describe how drugs are distributed and eliminated in the body through a system of ordinary differential equations, incorporating drug-specific and system-specific parameters \cite{book1}.  Accurate parameter estimation is crucial for modeling drug behavior in the body to ensure the effectiveness and safety of medications while inaccurate estimates can weaken pharmacokinetic predictions, slowing down drug development and clinical practices. 

The model problem used in this study to evaluate the predictive capacity of the PBPK-iPINN approach is adopted from the 4-compartment brain model included in Simcyp simulator (Certara Inc), a commercial software platform considered a gold standard for physiologically based pharmacokinetic (PBPK) modeling \cite{jamei2009simcyp, gaohua2016development}. For additional physiologically based pharmacokinetic brain models, we refer the reader to \cite{li2025mechanistic, wickramasinghe2025spatialcns, li2024mechanistic, wickramasinghe2025data}. The exiting commercial software packages for brain model simulations use traditional mathematical and statistical methodologies.  While PINNs have been widely applied to solve systems of ODEs, research on their use for inverse problems in physiologically based pharmacokinetic compartmental drug delivery modeling remains limited. By leveraging the complete physical information as prior knowledge, PINNs can be effectively trained using minimal or even no labeled data to serve as surrogate models for accurate solutions where the loss function measures the difference between the PINN outputs and the data. 

\begin{figure}[H]
\centering
\includegraphics[width=14cm]{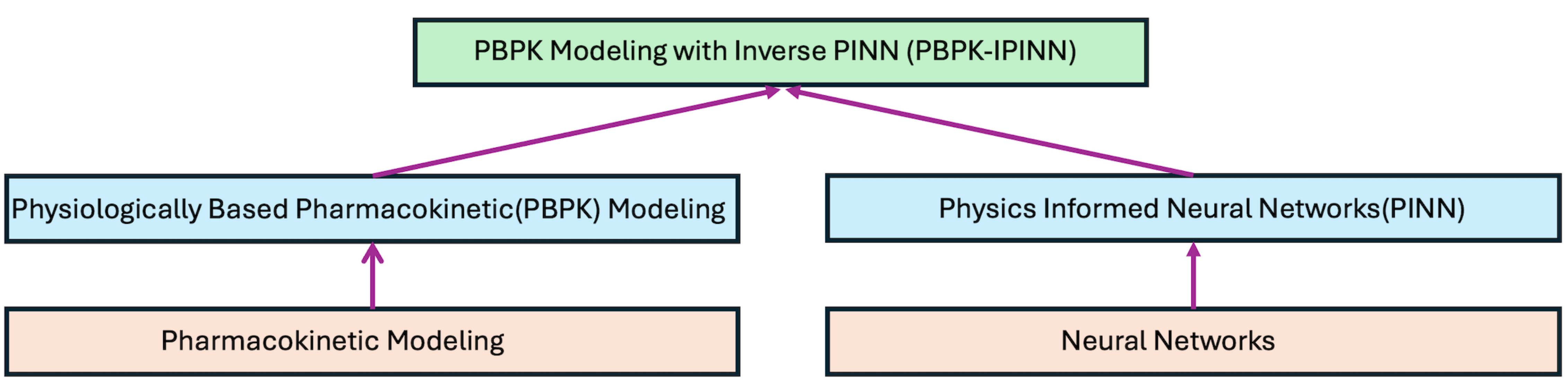}
\begin{tiny}
\caption{Advancing pharmacokinetic modeling with neural network.}\label{dp2}
\end{tiny}
\end{figure} 

The impact of our research is summarized in Figure (\ref{dp2}). While pharmacokinetic modeling has evolved to physiologically based pharmacokinetic approaches, and neural networks have advanced toward physics-informed neural networks, our work integrates these two developments by introducing an inverse physics-informed neural network framework to further advance PBPK modeling, thereby pushing PBPK modeling one step further.

Accurate parameter estimation of the model helps determine \( C_{\text{max}} \), \( T_{\text{max}} \), AUC, and half-life, which are crucial for understanding a drug's absorption, distribution, metabolism, and excretion. Here, \( C_{\text{max}} \) (the maximum concentration of a drug in the bloodstream) indicates the peak effectiveness and potential for side effects, while \( T_{\text{max}} \) (the time it takes to reach this peak) helps assess how quickly a drug acts. AUC (the area under the concentration-time curve) reflects the total drug exposure over time and is essential for evaluating the drug's bioavailability and therapeutic potential. Half-life (the time required for the drug concentration to reduce by half) informs dosing schedules and helps predict how long a drug will exert its effects. Thus, PBPK-iPINN approach as a promising tool for estimating drug parameters and predict the drug concentration profiles could assist drug developers and healthcare providers in developing more effective drugs and optimizing the use of current therapies to treat brain cancer.

The paper is organized as follows. In section 2 we present the system of ordinary differential equations and the schematic illustration of the 4-compartment brain model that mimic the drug transport in different regions of the human brain and it's existence and uniqueness of the forward problem. We further conduct  structural
and practical identifiability analyses to assess whether the model parameters can be
uniquely determined.  Section 3 is devoted to explain the inverse physics-informed neural network architecture for the four compartment brain model and a step by step guidance to implement a Python code of the PINN algorithm used through DeepXDE  library \cite{lu2021deepxde}. In Section 4, we present numerical results showing the predictive power of PBPK-iPINN and its validation. Finally, a conclusion is drawn that highlights impact of PBPK-iPINN in Section 5.

\section{Physiologically-Based Four Compartment Brain Model}\label{sec2}

The 4-compartment permeability limited brain (4Brain) model consisting of brain blood, brain mass, cranial and spinal cerebrospinal fluid (CSF) compartments has been developed and incorporated into a whole body physiologically-based pharmacokinetic (PBPK) model within the Simcyp Simulator. There are two approaches to understand the absorption, distribution, metabolism, and excretion (ADME) of drugs in pharmacology called compartment and non compartment modeling. In this study we used compartment modeling which assumes that the body can be represented by a series of interconnected compartments (e.g., central and peripheral compartments) where the drug is distributed. These compartments represent different tissues or groups of tissues with similar drug distribution characteristics. The figure (\ref{brain4}) illustrate the how drugs are  distributed in and out of each compartment of the brain.

\subsection{System of Differential Equations}

In compartmental pharmacokinetic modeling, the body is divided into compartments (e.g., blood, tissues) where drug movement follows mass balance principle \cite{beumer2006mass}.
Mass balance ensures we keep track of every bit of the drug it doesn't disappear or show up out of nowhere unless we include a way for that to happen (like the body breaking it down). This means if drug leaves one compartment, it must enter another (or be eliminated). The differential equations model how drug amounts change over time. For any compartment with drug amount \( Y(t) \), the mass balance principle states:
\begin{equation}
    \frac{dY}{dt} = \sum (\text{Input Rates}) - \sum (\text{Output Rates})
\end{equation}
where, Inputs (positive terms) represent drug entering the compartment (e.g., absorption, infusion, transfer from another compartment) while Outputs (negative terms) represent drug leaving the compartment (e.g., elimination, distribution to other compartments). Thus, following the mass conservation law the drug disposition in 4 compartment brain model can be described by the following system of differential equations \cite{gaohua2016development}.  The equations (\ref{eqn1}), (\ref{eqn2}), (\ref{eqn3}), and (\ref{eqn4}) represent the rate of change of drug concentration of brain blood ($C_{bb}$), brain mass ($C_{bm}$), cranial CSF ($C_{ccsf}$), and spinal CSF ($C_{scsf}$) respectively where $C_{art}$ denotes the arterial blood concentration(mg/L) which is an exogenous input to the system.

\begin{figure}[]
\centering
\includegraphics[width=12cm]{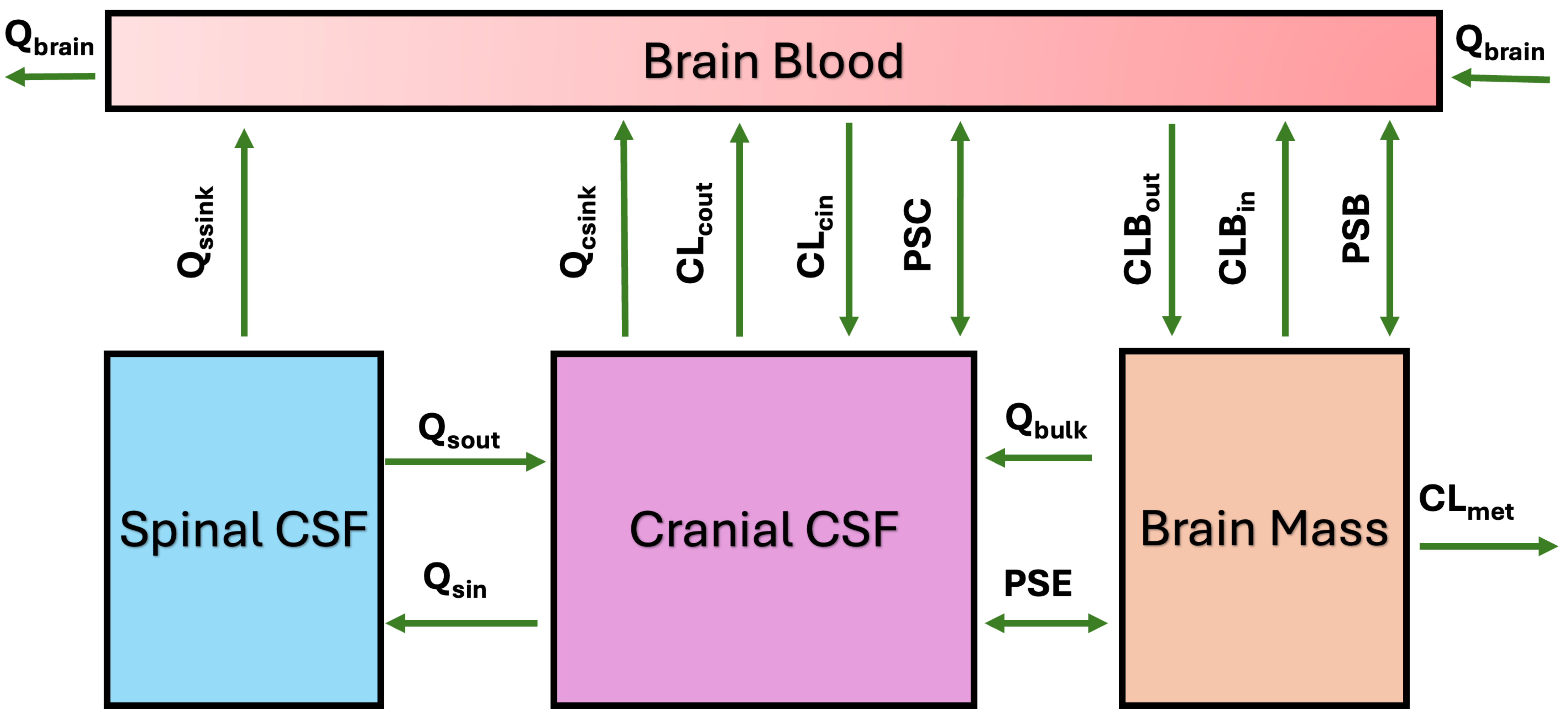}
\begin{tiny}
\caption{Schematic illustration of the 4 compartment brain model.}\label{brain4}
\end{tiny}
\end{figure}

\noindent \textbf{\textit{Brain blood compartment:}}

\begin{equation} \label{eqn1}
\begin{split}
V_{bb} \frac{dC_{bb}}{dt} & = Q_{brain}(C_{art}-C_{bb}) + PSB(\lambda_{bm}fu_{bm}C_{bm}-\lambda_{bb}fu_{bb}C_{bb})   \\ & +  CLB_{in}fu_{bb}C_{bb} + CLB_{out}fu_{bm}C_{bm} 
\\ & + PSC(\lambda_{ccsf}fu_{ccsf}C_{ccsf}-\lambda_{bb}fu_{bb}C_{bb}) - CLC_{in}fu_{bb}C_{bb} 
\\ & + CLC_{out}fu_{ccsf}C_{ccsf} + Q_{csink}C_{ccsf} 
+ Q_{ssink}C_{scsf}
\end{split}
\end{equation}
\noindent \textbf {\textit{Brain mass compartment:}}
\begin{equation} \label{eqn2}
\begin{split}
\begin{aligned}
 V_{bm} \frac{dC_{bm}}{dt} &= PSB(\lambda_{bb}fu_{bb}C_{bb}-\lambda_{bm}fu_{bm}C_{bm})  +CLB_{in}fu_{bb}C_{bb} \\& - CLB_{out}fu_{bm}C_{bm} - Q_{bulk}C_{bm}\\
&  + PSE(\lambda_{ccsf}fu_{ccsf}C_{ccsf}-\lambda_{bm}fu_{bm}C_{bm})  - CL_{met}C_{bm} 
\end{aligned}
\end{split}
\end{equation}
\noindent \textbf {\textit{Cranial CSF compartment:}}
\begin{equation} \label{eqn3}
\begin{split}
\begin{aligned}
 \qquad V_{ccsf} \frac{dC_{ccsf}}{dt} 
 &= PSC(\lambda_{bb}fu_{bb}C_{bb}-\lambda_{ccsf}fu_{ccsf}C_{ccsf}) 
+CLC_{in}fu_{bb}C_{bb}  \\&  - CLC_{out}fu_{ccsf}C_{ccsf}  - Q_{sout}C_{scsf} 
 \\&  + PSE(\lambda_{bm}fu_{bm}C_{bm}-\lambda_{ccsf}fu_{ccsf}C_{ccsf})  - Q_{sin}C_{ccsf} 
 \\& - Q_{csink}C_{ccsf}
\end{aligned}
\end{split}
\end{equation}
\noindent \textbf {\textit{Spinal CSF compartment:}}
\begin{equation} \label{eqn4}
\begin{split}
\begin{aligned}
\qquad V_{scsf} \frac{dC_{scsf}}{dt} &= Q_{sin}C_{ccsf} - Q_{sout}C_{scsf} - Q_{ssink}C_{scsf} \hspace{8cm}
\end{aligned}
\end{split}
\end{equation}
The system specific parameters  were derived from the existing Simcyp virtual cancer patient population and are listed in the Table (\ref{tab:1}). These parameters are related to the physiological and biochemical characteristics of the human body. They influence how the body handles the drug, and they can vary from person to person. The drug-specific parameters can be found in the literature \cite{gaohua2016development} and from the Simcyp simulator. Drug specific parameters are listed in Table (\ref{tab:2}). These parameters are intrinsic to the drug itself and dictate how it behaves in the body. The drug specific parameters are specific to an oral dose of 10 mg of abemaciclib drug which is a targeted cancer therapy used to treat certain types of cancer.  They are crucial in designing optimal dosing regimens and understanding drug interactions. These parameter values are considered as reference values to validate our parameter estimation approach.

\begin{table}[htbp]
\centering
\caption{System specific parameters for abemaciclib in 4-compartment brain model}
\label{tab:1}
\footnotesize
\setlength{\tabcolsep}{10pt}
\renewcommand{\arraystretch}{1.3}
\begin{tabular}{|>{\RaggedRight}p{0.18\linewidth}|>{\RaggedRight}p{0.5\linewidth}|c|}
\hline
\textbf{Parameter} & \textbf{Description (unit)} & \textbf{Value} \\
\hline
$V_{bb}$ & Brain blood volume (L) & 0.064952435 \\
\hline
$V_{bm}$ & Brain mass volume (L) & 1.104115461 \\
\hline
$V_{ccsf}$ & Cranial CSF volume (L) & 0.103984624 \\
\hline
$V_{scsf}$ & Spinal CSF volume (L) & 0.025996156 \\
\hline
$Q_{brain}$ & Blood/CSF flow (L/h) & 38.0 \\
\hline
$Q_{csink}$ & Cranial CSF absorption rate (L/h) & 0.01277633 \\
\hline
$Q_{ssink}$ & Spinal CSF absorption rate (L/h) & 0.007761342 \\
\hline
$Q_{bulkBC}$ & Bulk flow from brain mass to cranial CSF (L/h) & 0.005164106 \\
\hline
$Q_{bulkCB}$ & Bulk flow from cranial CSF to brain mass (L/h) & 0.005164106 \\
\hline
$Q_{sout}$ & CSF flow: spinal to cranial CSF (L/h) & 0.007489995 \\
\hline
$Q_{sin}$ & CSF flow: cranial to spinal CSF (L/h) & 0.015251337 \\
\hline
$PSB$ & Passive permeability-surface area product at the BBB (L/h) & 135.0 \\
\hline
$PSC$ & Passive permeability-surface area product at the blood–cranial CSF barrier (L/h) & 67.5 \\
\hline
$PSE$ & Passive permeability-surface area of brain-CSF barrier (L/h) & 300.0 \\
\hline
\end{tabular}
\end{table}

\begin{table}[htbp]
\centering
\caption{Drug-specific parameters for abemaciclib in 4-compartment brain model}
\label{tab:2}
\footnotesize
\setlength{\tabcolsep}{13pt} 
\renewcommand{\arraystretch}{1.4}
\begin{tabular}{|>{\RaggedRight}p{0.18\linewidth}|>{\RaggedRight}p{0.5\linewidth}|c|}
\hline
\textbf{Parameter} & \textbf{Description (unit)} & \textbf{Value} \\
\hline
$CLB_{in}$ & Clearance of active uptake transporter on the BBB (L/h) & 0.0 \\
\hline
$CLB_{out}$ & Clearance of active efflux transporter on the BBB (L/h) & 110.0 \\
\hline
$CLC_{in}$ & BCSFB uptake transporter (L/h) & 11.9 \\
\hline
$CLC_{out}$ & BCSFB efflux transporter (L/h) & 0.0 \\
\hline
$CL_{met}$ & Metabolic clearance due to brain enzymes (L/h) & 0.0 \\
\hline
$fu_{bb}$ & Drug unbound fraction in the brain blood & 0.125 \\
\hline
$fu_{bm}$ & Drug unbound fraction in the brain mass & 0.044 \\
\hline
$fu_{ccsf}$ & Drug unbound fraction in the Cranial CSF & 1.0 \\
\hline
$\lambda_{bb}$ & Unionization fraction in the brain blood & 0.033 \\
\hline
$\lambda_{bm}$ & Unionization fraction in the brain parenchyma & 0.017 \\
\hline
$\lambda_{ccsf}$ & Unionization fraction in the cranial CSF & 0.026 \\
\hline
\end{tabular}
\end{table}

\subsection{Existence and Uniqueness}\label{subs2}
Once the model is built we check the existence and uniqueness \cite{apostol1969multi}.The model problem described by equations (\ref{eqn1}), (\ref{eqn2}), (\ref{eqn3}), and (\ref{eqn4}) can be written in the following general form for non-autonomous linear systems of ODE as shown in the equations (\ref{exis}) and (\ref{exis-init}).
\begin{equation} \label{exis}
\begin{aligned}
 Y'(t) = \bm{A}(\theta)Y + \bm{G}(t;\theta), \quad Y \in \mathbb{R}^{n}, \quad \bm{A} \in R^{ n\times n}, \quad \bm{g} \in R^{n}
\end{aligned}
\end{equation}
\begin{equation} \label{exis-init}
\begin{aligned}
Y(t_{0}) = Y_{0}, \quad (t_{0}, Y_{0}) \in I \times \mathbb{R}^{k}, \quad I \subseteq \mathbb{R}
\end{aligned}
\end{equation}
where, 
$\theta \in \Theta \subseteq \mathbb{R}^p$ is a vector of constant parameters.
\begin{align*}
A &= \begin{bmatrix}
    \theta_{11} & \theta_{12} & \theta_{13} & \theta_{14} \\
    \theta_{21} & \theta_{22} & \theta_{23} & \theta_{24} \\
    \theta_{31} & \theta_{32} & \theta_{33} & \theta_{34} \\
    \theta_{41} & \theta_{42} & \theta_{43} & \theta_{44}
\end{bmatrix},
&
G(t;\theta) &= \begin{bmatrix}
    g_1(t;\theta) \\
    0 \\
    0 \\
    0
\end{bmatrix}
\end{align*}

\begin{equation*}
Y = \begin{bmatrix}
C_{bb} \\
C_{bm} \\
C_{ccsf} \\
C_{scsf}
\end{bmatrix},
\quad \text{} \quad
 Y'(t) = \frac{dY}{dt}, 
\quad \text{and} \quad
g_1(t;\theta) = Q_{brain}C_{art}
\end{equation*}

\begin{align*}
\theta_{11} &= -(Q_{brain} + PSB\lambda_{bb}fu_{bb} - CLB_{in}fu_{bb} + PSC\lambda_{bb}fu_{bb} + CLC_{in}fu_{bb}) / V_{bb} \\
\theta_{12} &= (PSB\lambda_{bm}fu_{bm} + CLB_{out}fu_{bm}) / V_{bb} \\
\theta_{13} &= (PSC\lambda_{ccsf}fu_{ccsf} + CLC_{out}fu_{ccsf} + Q_{csink}) / V_{bb} \\
\theta_{14} &= 0 \\
\theta_{21} &= (PSB\lambda_{bb}fu_{bb} + CLB_{in}fu_{bb}) / V_{bm} \\
\theta_{22} &= -(PSB\lambda_{bm}fu_{bm} + CLB_{out}fu_{bm} + Q_{bulk} + PSE\lambda_{bm}fu_{bm} + CL_{met}) / V_{bm} \\
\theta_{23} &= (PSE\lambda_{ccsf}fu_{ccsf}) / V_{bm} \\
\theta_{24} &= 0 \\
\theta_{31} &= (PSC\lambda_{bb}fu_{bb} + CLC_{in}fu_{bb}) / V_{ccsf} \\
\theta_{32} &= (PSE\lambda_{bm}fu_{bm}) / V_{ccsf} \\
\theta_{33} &= -(PSC\lambda_{ccsf}fu_{ccsf} + CLC_{out}fu_{ccsf} + Q_{sout} + PSE\lambda_{ccsf}fu_{ccsf} + Q_{sin} + Q_{csink}) / V_{ccsf} \\
\theta_{34} &= -Q_{sout} / V_{ccsf} \\
\theta_{41} &= 0 \\
\theta_{42} &= 0 \\
\theta_{43} &= Q_{sin} / V_{scsf} \\
\theta_{44} &= -(Q_{sout} + Q_{ssink}) / V_{scsf}
\end{align*}

\begin{theorem}[Existence and Uniqueness of the 4 Compartment Brain Model]
Let, A be an $n\times n$ constant matrix and let G be an n-dimensional vector
function continuous on an interval I $\subseteq \mathbb{R}$. Pick $t_0 \in I$. Then the initial value problem 
\begin{equation}\label{tm1}
Y'(t) = AY(t) + G(t), \qquad Y(t_0) = Y_{0}
\end{equation}

\textit{has a unique solution on $I$, which is}
\begin{equation}\label{tm2}
Y(x) = e^{(x-t_0)A}Y_{0} + e^{xA} \int_{t_0}^x e^{-tA} G(t) \, dt.
\end {equation}
\end{theorem}

\begin{proof} The derivative of the solution Y(x) is

\begin{align*}
Y'(x) &= Ae^{(x - t_0)A}Y_0 + Ae^{xA} \int_{t_0}^x e^{-tA}G(t) \, dt + e^{xA}e^{-xA}G(x) \\
      &= AY(x) + G(x)
\end{align*} so it satisfies the differential equation. Also, setting $x = t_0$ in the equation \ref{tm2} gives $Y(t_0) = Y_{0}$ so it also satisfies the initial condition. This guarantee the existence of a solution. To prove the uniqueness, suppose that $V(x)$ is another solution to the differential equation such that
\[
V'(x) = AV(x) + G(x), \qquad V(t_0) = V_{0},
\]
Then $P(x) = Y(x) - V(x)$ is a solution to the homogeneous equation such that 

\begin{equation}\label{hom1}
P'(x) = AP(x), \qquad P(t_0) = 0
\end{equation}
This is a first-order linear homogeneous system with constant coefficient matrix \( A \). The general solution to the equation (\ref{hom1}) can be written as $P(x) = e^{(x - t_0)A} C,$ where \( C \) is a constant vector (or matrix) determined by the initial condition. Using the initial condition \( P(t_0) = 0 \), we substitute:$
P(t_0) = e^{(t_0 - t_0)A} C = e^{0A} C = I C = C = 0.$ Thus we get The unique solution to equation (\ref{hom1}) is P(x) = 0. This concludes that $V(x) =Y(x)$ hence guarantees the uniqueness of the equation (\ref{tm1}). 
\end{proof}

\subsection{Parameter Identifiability Analysis
}

In this section, to complete the system identification workflow, we perform structural and practical identifiability analyses to assess whether the model parameters can be uniquely determined from the available data.

\noindent\textbf{Structural Identifiability Analysis}

We consider the following general form of the mathematical model for parameter estimation:

\begin{equation}\label{parai}
\quad \dot{x}(t,p) = f(x(t), u(t), p), \quad y(t,p) = g(x(t), p), \quad x_0 = x(t_0, p)
\end{equation}

where $f$ and $g$ are vector functions of their arguments, $p \in \mathbb{R}^p$ is a $p$-dimensional vector of parameters, $x(t) \in \mathbb{R}^n$ is the $n$-dimensional state variable vector, $u(t) \in \mathbb{R}^r$ is the $r$-dimensional input vector, and $y(t) \in \mathbb{R}^m$ are the $m$-dimensional measured outputs. A parameter $p_i$ is said to be identifiable if the following equation holds true:
\begin{equation}\label{identi}
y(t, p) = y(t, p^*) \quad \Rightarrow \quad p_i = p_i^*
\end{equation}

where $p^*$ is an alternative parameter vector.  

\begin{itemize}
    \item If Eq.(\ref{identi}) holds for any $p_i^*$, then $p_i$ is said to be \textbf{globally identifiable}.
    \item If Eq.(\ref{identi}) holds for a neighborhood of $p_i^*$, then $p_i$ is said to be \textbf{locally identifiable}.
    \item If Eq.(\ref{identi}) does not hold true for any $p_i$ locally or globally, then $p_i$ is said to be \textbf{structurally unidentifiable}.
\end{itemize}

There are multiple methods that can be used to perform structural identifiability analysis \cite{prac1, prac2}. In this section, we only test for the local identifiability of the system and, for convenience, we will refer to a system as being identifiable when it is structually locally identifiable. We use the Julia library StructuralIdentifiability \cite{prac3} to test for structural identifiability of the model. The existing algorithms implemented in the library require both f and g to be rational functions, which are fractions of polynomials. We specify the parametric ODE model using the @ODEmodel macro. x’(t) is the derivative of state variable x(t), which is assumed to be unknown if not specified otherwise. y(t) defines the output variable which is assumed to be given. Figure (\ref{identifiability}) shows the Julia code to perform local structural identifiability. The last line of the code tests the local identifiability of the model. 

\begin{figure}[H]
\centering
\includegraphics[width=15cm]{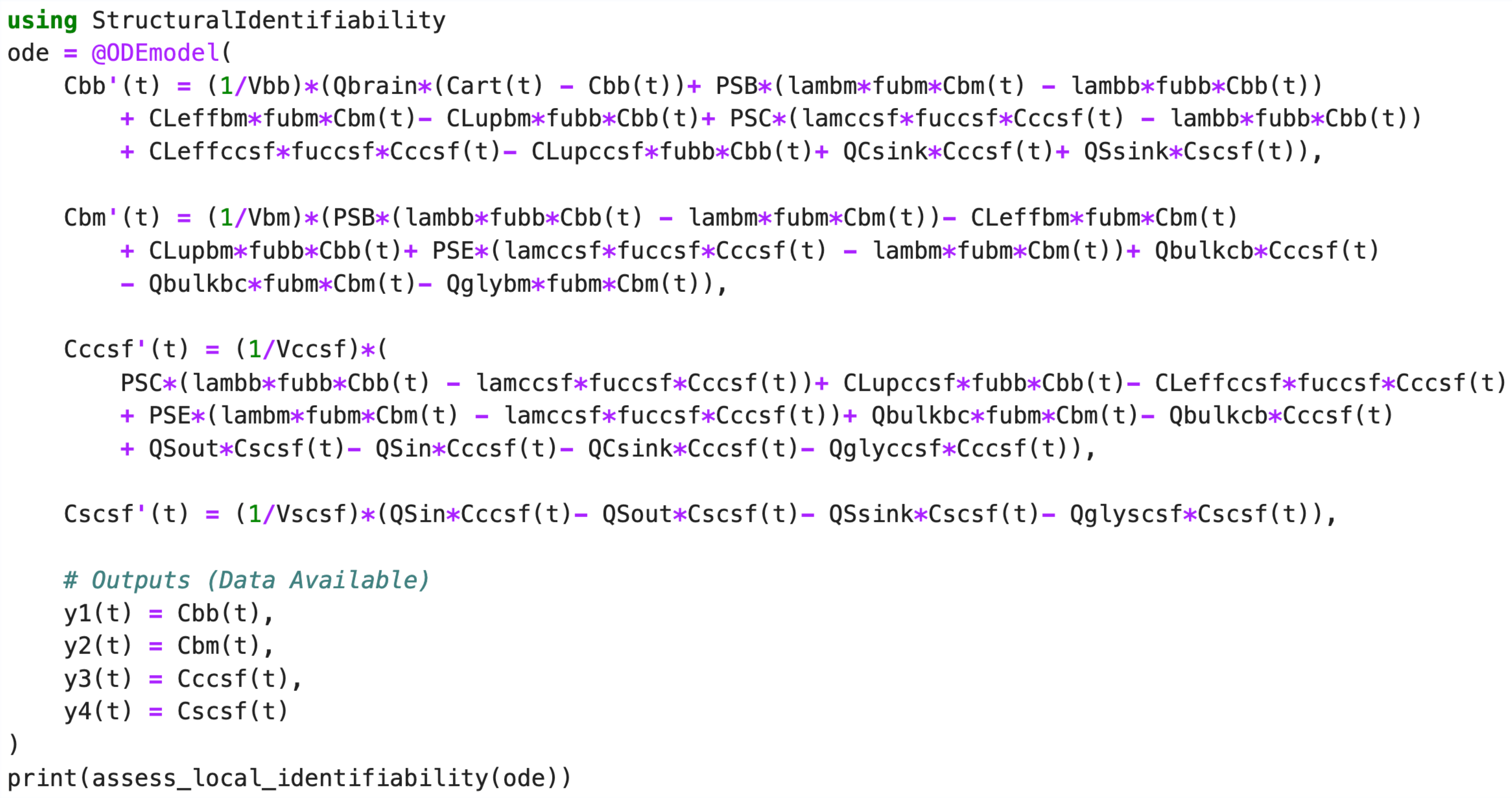}
\begin{tiny}
\caption{Julia code to  perform local structural identifiability analysis}\label{identifiability}
\end{tiny}
\end{figure}

\begin{table}[h!]
\centering
\caption{Local structural identifiability result of the 4-compartmental brain model}
\label{tab:3}
\begin{tabular}{>{\raggedright\arraybackslash}p{3cm} p{2cm} p{2cm} p{2cm} p{2cm} p{2cm}}
\hline
Parameter & Case1 & Case2 & Case3 & Case4 & Case5 \\
\hline
Vbb & \texttimes & \texttimes & \texttimes & \checkmark & \checkmark \\
Vbm & \texttimes & \texttimes & \checkmark & \checkmark & \checkmark \\
Vccsf & \texttimes & \texttimes & \checkmark & \checkmark & \checkmark \\
Vscsf & \texttimes & \texttimes & \checkmark & \checkmark & \checkmark \\
Qbrain & \texttimes & fixed & fixed & fixed & fixed \\
Qbulkbc & \texttimes & fixed & fixed & fixed & fixed \\
Qbulkcb & \texttimes & fixed & fixed & fixed & fixed \\
QCsink & \texttimes & fixed & fixed & fixed & fixed \\
QSsink & \texttimes & \checkmark & \checkmark & \checkmark & \checkmark \\
QSin & \texttimes & \texttimes & fixed & fixed & fixed \\
QSout & \texttimes & \texttimes & fixed & fixed & fixed \\
Qglybm & \texttimes & \texttimes & fixed & fixed & fixed \\
Qglyccsf & \texttimes & \texttimes & fixed & fixed & fixed \\
Qglyscsf & \texttimes & \texttimes & \texttimes & fixed & fixed \\
PSB & \texttimes & \texttimes & \texttimes & fixed & fixed \\
PSC & \texttimes & \texttimes & \texttimes & fixed & fixed \\
PSE & \texttimes & \texttimes & \texttimes & fixed & fixed \\
CLeffbm & \texttimes & \texttimes & \texttimes & \texttimes & fixed \\
CLupbm & \texttimes & \texttimes & \texttimes & \texttimes & fixed \\
CLeffccsf & \texttimes & \texttimes & \texttimes & \texttimes & fixed \\
CLupccsf & \texttimes & \texttimes & \texttimes & \texttimes & fixed \\
fubb & \texttimes & \texttimes & \texttimes & \texttimes & \checkmark \\
fubm & \texttimes & \texttimes & \texttimes & \texttimes & \checkmark \\
fuccsf & \texttimes & \texttimes & \texttimes & \texttimes & \checkmark \\
lambb & \texttimes & \texttimes & \texttimes & \texttimes & \checkmark \\
lambm & \texttimes & \texttimes & \texttimes & \texttimes & \checkmark \\
lamccsf & \texttimes & \texttimes & \texttimes & \texttimes & \checkmark \\
\hline
\end{tabular}
\end{table}

Due to the nonlinear nature of the model, parameter identifiability is sometimes challenging. To improve identifiability, we systematically fixed certain parameters, which allowed more parameters to become locally identifiable, as demonstrated in Cases 1–5 of Table (\ref{tab:3}). In all five cases, data were provided for every state variable.

\noindent\textbf{Practical identifiability analysis}

Structural identifiability examines only the mathematical structure of a system of ODEs, the model inputs, and the outputs that are measured. It does not take into account how these outputs are actually measured or the experimental data itself. Practical identifiability analysis, on the other hand, evaluates whether the model parameters can be identified, at least locally, using both the system of ODEs and the available experimental data. It is important to note that a model must be structurally identifiable before it can be practically identifiable. Although several approaches exist for conducting practical identifiability analysis \cite{prac4, prac5}, this study focuses specifically on the bootstrapping method. We utilize MONOLIX software to perform bootstrap analysis and present the correlation matrix of the parameters and the confidence intervals for estimated parameters in Example 2.

\section{Methodology of Inverse PINN}\label{sec:3}
Physics-Informed Neural Networks are deep neural networks that can be trained to solve forward and inverse differential equation problems while respecting the physical laws given by the differential equations \cite{hariri2025physics,raissi2019physics}. This is achieved by incorporating a physics-based term into the loss function during optimization procedure in the training process. The convergence is achieved by minimizing a loss function which expression is based on the mean squared error. Finding optimal set of  parameters is achieved by solving an optimization problem using a gradient algorithm that relies on automatic differentiation to back-propagate gradients through the network \cite{baydin2018automatic}.

\subsection{PINN Architecture of 4-Compartmental Brain Model}
The aim of this section is to present the PINN technique for estimating parameters and then solving the system of differential equations given the observed data. Figure (\ref{nn}) shows a neural network architecture starts from an input time vector. Then, as the output of neural network model we obtain the solution of the system of ODEs $Y(t;\theta) = (C_{bb}, C_{bm}, C_{ccsf}, C_{scsf})$. The output goes to an optimization block where it minimizes the data loss, initial condition loss  as well as the ODE loss and updates the neural network parameters. We compute the derivatives of the neural network outputs with respect to inputs using automatic differentiation (AD) via PyTorch's backward propagation and chain rule. We employ adaptive weights for the data loss ($\lambda_{Data}$), initial condition loss ($\lambda_{IC}$) and ODE loss ($\lambda_{ODE}$) that can be imposed to train simultaneously with the neural network parameters $\theta$. However, we manually fixed these parameters $\lambda_{Data}$, $\lambda_{IC}$, and $\lambda_{ODE}$ to certain values to balance the individual losses. The system parameters $p$ are simultaneously trained with the neural network parameters $\theta$ as external trainable variables. 

\begin{figure}[H]
\centering
\includegraphics[width=14cm]{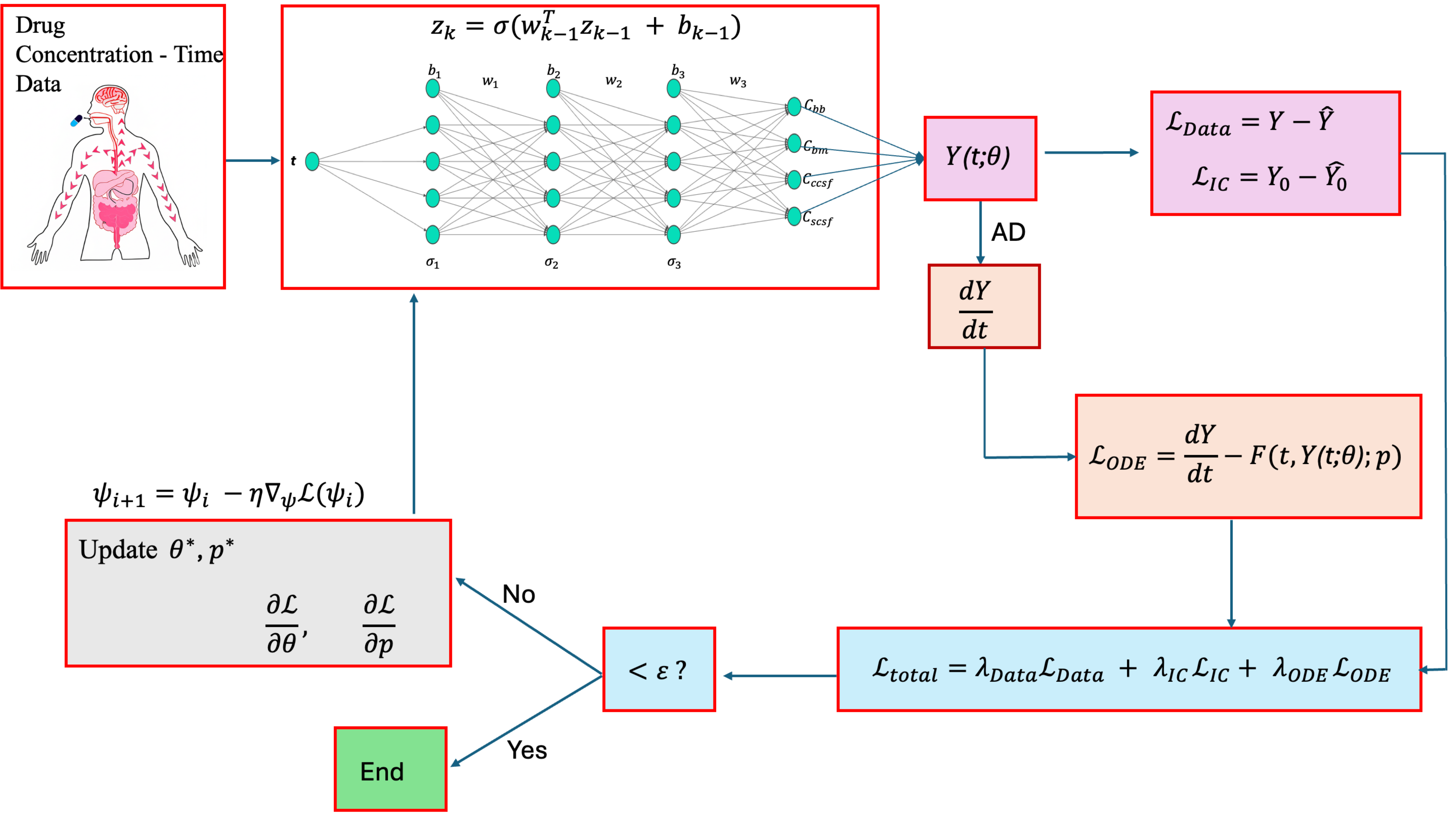}
\begin{tiny}
\caption{The inverse physics-informed neural network model starts with time as inputs and then the outputs goes to an optimization block where it minimizes the total loss  by optimizing the parameters $\psi$ which includes system parameters $p$ and neural network parameters $\theta$.
}\label{nn}
\end{tiny}
\end{figure}

Mathematically, a deep neural network can be represented as a hierarchical nonlinear mapping that transforms an input into an output through a series of parameterized operations across its layers. These layers are also referred to as hidden layers because their outputs are not visible to the external environment. The closed form of the solution $\mathbf{Y}(t; \boldsymbol{p})$ can be obtained by training a deep neural network which is a type of neural network that has multiple layers of artificial neurons between the input and output layers. In this study, we use fully connected neural networks (NNs) to model the velocity field. These NNs are built by connecting multiple layers of artificial neurons, where each layer transforms its input in two steps: \textbf{(1) Linear Transformation}, where the input from the previous layer ($\mathbf{z}_{k-1}$) is multiplied by a weight matrix ($\mathbf{W}_{k-1}$) and shifted by a bias vector ($\mathbf{b}_{k-1}$), expressed as; $$\text{Linear output} = \mathbf{W}_{k-1}^T \mathbf{z}_{k-1} + \mathbf{b}_{k-1}$$ and \textbf{(2) Nonlinear Activation}, where the result is passed through an element-wise activation function $\sigma$ to introduce nonlinearity: $$\mathbf{z}_k = \sigma(\text{Linear output})$$. The weights ($\mathbf{W}$) and biases ($\mathbf{b}$) are optimized during training using the \textit{ADAM} algorithm, a variant of gradient descent.  We use the following  activation functions in this study to train the neural network and for comparison purpose\cite{wang2023learning, maczuga2023influence}. 
\vspace{0.2cm}

\noindent Hyperbolic Tangent (Tanh):
$$\sigma(x) = \tanh(x) = \frac{e^x - e^{-x}}{e^x + e^{-x}}, \quad x \in (-\infty, +\infty), \quad \sigma(x) \in (-1, 1)$$

\noindent Sigmoid (Logistic):
$$\sigma(x) = \frac{1}{1 + e^{-x}}, \quad x \in (-\infty, +\infty), \quad \sigma(x) \in (0, 1)$$

\noindent Rectified Linear Unit (ReLU):
$$\sigma(x) = \max(0, x), \quad x \in (-\infty, +\infty), \quad \sigma(x) \in [0, +\infty)$$

\noindent  Periodic (Sine):
$$
\sigma(x) = \sin(\omega x), \quad \omega>0, \quad x \in (-\infty, +\infty), \quad \sigma(x) \in [-1, 1]
$$

Figure (\ref{ac}) shows the graphs of the activation functions used in this study: Rectified Linear Unit (ReLU), Sigmoid, Tanh, and Sine. A few other widely used activation functions include ELU, GELU, SELU, SiLU, and Swish; see \cite{jagtap2020locally} for further details. It is essential that the activation function be non-linear. If a linear activation were used, the network would reduce to a composition of linear transformations, which itself is equivalent to a single linear mapping, thereby severely restricting its approximation capacity. As a simple illustrative example, we consider a neural network with three hidden layers to show how the network processes the input. Let the input layer receive temporal coordinates $\mathbf{t} = (t_1, \dots, t_d) \in \mathbb{R}^{d+1}$. Then fully connected three hidden layers with $n_1$, $n_2$, and $n_3$ neurons, respectively can be presented as, 
\begin{equation*}
 \mathbf{z}_{1}= \sigma(\mathbf{W}_1 \mathbf{t} + \mathbf{b}_1) 
\quad \rightarrow \quad
\mathbf{z}_{2} = \sigma(\mathbf{W}_2  \mathbf{z}_{1} + \mathbf{b}_2) 
\quad \rightarrow \quad
\mathbf{z}_{3} = \sigma(\mathbf{W}_3  \mathbf{z}_{2} + \mathbf{b}_3)
\end{equation*}

\noindent for a given activation function $\sigma$. Then the output layer produces the solution vector $\mathbf{Y}(t; \boldsymbol{\theta},p) \in \mathbb{R}^4$ for the ODE system.
    \begin{align*}
    \text{Output:} & \quad \mathbf{Y}(t; \boldsymbol{\theta},p) = \mathbf{W}_4 \mathbf{z}_{3} + \mathbf{b}_4
    \end{align*}
\noindent where, weight matrices are defined as:
\[
\mathbf{W}_1 \in \mathbb{R}^{n_1 \times d}, \quad
\mathbf{W}_2 \in \mathbb{R}^{n_2 \times n_1}, \quad
\mathbf{W}_3 \in \mathbb{R}^{n_3 \times n_2}, \quad
\mathbf{W}_4 \in \mathbb{R}^{4 \times n_3},
\]
and the bias vectors are given by:
\[
\mathbf{b}_1 \in \mathbb{R}^{n_1}, \quad
\mathbf{b}_2 \in \mathbb{R}^{n_2}, \quad
\mathbf{b}_3 \in \mathbb{R}^{n_3}, \quad
\mathbf{b}_4 \in \mathbb{R}^4.
\]

\noindent The complete set of trainable parameters includes both the network parameters $
\boldsymbol{\theta} = \bigcup_{i=1}^{4} \{ \mathbf{W}_i, \mathbf{b}_i \}
$ and the system parameters $p \in \mathbb{R}^k$, where $p$ represents additional physical or model parameters that need to be learned during training as external trainable variables.

\begin{figure}[H]
\centering
\includegraphics[width=14cm]{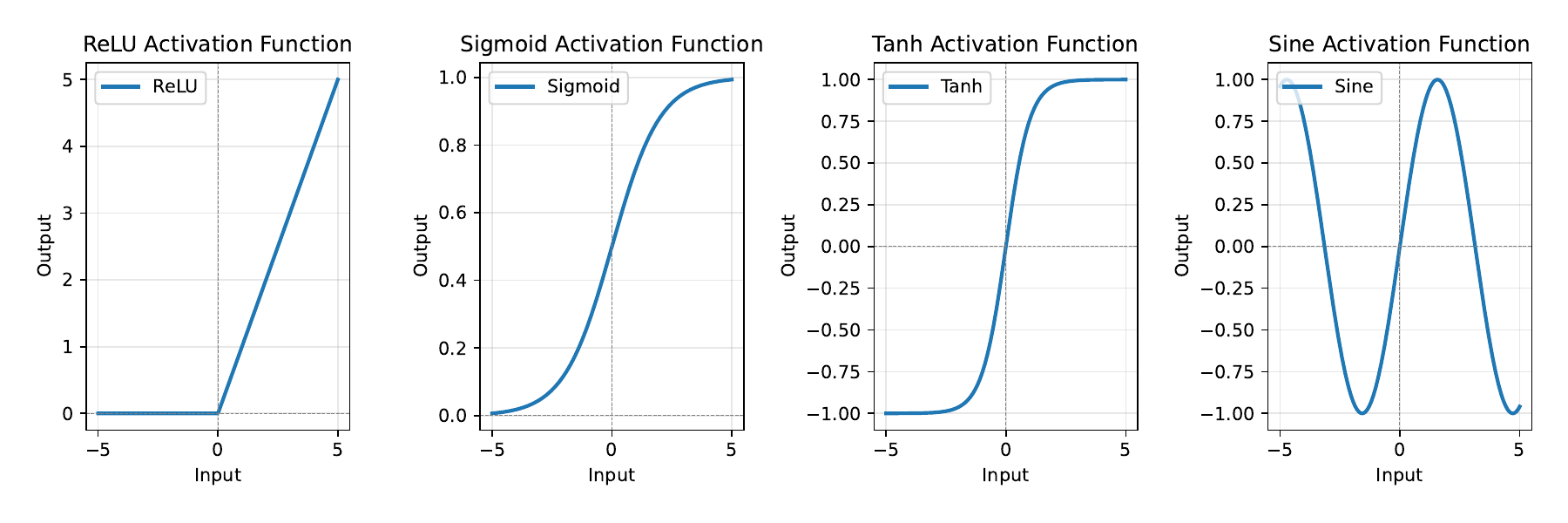}
\begin{tiny}
\caption{Activation functions.}\label{ac}
\end{tiny}
\end{figure}

\subsection{Augmented Loss Function and Its Minimization} The loss function in PINNs is designed to enforce both data fidelity and physical consistency by combining multiple objective terms. The data loss ($\mathcal{L}_{\text{data}}$) ensures that the neural network's predictions align with observed measurements, minimizing discrepancies at known data points. Data loss is calculated by the sum of the  square differences of the predicted concentrations and the observed data of the four compartments $C_{bb}$, $C_{bm}$, $C_{ccsf}$, and $C_{scsf}$. The ODE loss ($\mathcal{L}_{\text{ODE}}$), weighted by $\lambda_{ODE}$, penalizes deviations from the governing physical laws, embedding the underlying dynamics directly into the learning process. PDE loss makesure to fulfill the mass balance law described by the four-compartmental brain model.  Finally, the initial condition loss ($\mathcal{L}_{\text{IC}}$), scaled by $\lambda_{IC}$, guarantees that the solution adheres to prescribed initial constraints. By optimizing this composite loss function, the network not only interpolates sparse data but also generalizes as a physics-compliant surrogate model, robust even in regions where measurements are unavailable. Then, we define the total loss, $\mathcal{L}$, as follows:
\[
\mathcal{L} = \lambda_{Data} \mathcal{L}_{\text{data}} + \lambda_{ODE} \mathcal{L}_{\text{ODE}} + \lambda_{IC}  \mathcal{L}_{\text{IC}}
\]
where:
\begin{align*}
    \mathcal{L}_{\text{data}} &= \frac{1}{N_d} \sum_{i=1}^{N_d} \|\mathbf{y}(t_i) - \mathbf{y}_i\|^2 \\
    \mathcal{L}_{\text{ODE}} &= \frac{1}{N_c} \sum_{j=1}^{N_c} \sum_{k=1}^4 \|\frac{dY_{k}(t_{j})}{dt}- F_k(t_j)\|^2 \\
    \mathcal{L}_{\text{IC}} &= \frac{1}{N_{\text{IC}}} \sum_{l=1}^{N_{\text{IC}}} \|\mathbf{y}(0) - \mathbf{y}_0\|^2.
\end{align*}
\noindent The optimal parameters $\psi^*$ are obtained by minimizing the total loss function:
\begin{equation}
\psi^* = \underset{\psi}{\arg\min} \, \mathcal{L}(\psi)
\end{equation}
\noindent where, $\psi$ represents all trainable parameters of the neural network including $\theta$ and $p$.  $\mathcal{L}(\psi)$ is the composite loss function defined as 
\begin{equation*}
\mathcal{L} = \lambda_{Data}\overbrace{\frac{1}{N_d} \sum \|\mathbf{y} - \mathbf{y}_{\text{data}}\|^2}^{\text{Data}} + \lambda_{ODE} \overbrace{\frac{1}{N_c} \sum \|\frac{dY}{dt} - \mathbf{F}\|^2}^{\text{ODE}} + \lambda_{IC} \overbrace{\frac{1}{N_{\text{IC}}} \sum \|\mathbf{y}(0) - \mathbf{y}_0\|^2}^{\text{IC}}.
\end{equation*}

\noindent A gradient descent algorithm is used until convergence towards the minimum is obtained for a predefined accuracy (or a given maximum iteration number) as
\begin{equation}
    \psi_{i+1} = \psi_i - \eta \nabla_{\psi} \mathcal{L}(\psi_i),
    \label{eq:gradient_descent}
\end{equation}
for the $i$-th iteration (also called epoch in the literature), leading to $\psi^* = \arg\min_{\psi} \mathcal{L}(\psi)$, where $\eta$ is known as the learning rate parameter. In this work, we choose the well-known Adam optimizer. A standard automatic differentiation technique is necessary to compute derivatives (i.e., $\nabla_{\psi}$) with respect to the neural network parameters (e.g., weights and biases) of the model \cite{raissi2019physics}. An important feature of this architecture is its flexibility, which enables the simultaneous optimization of unknown parameters along with the network’s weights and biases. Thus the system parameters $p$ are trained as external trainable variable. Optimal choice of values for parameters ($\lambda_{Data}, \lambda_{ODE}$ and $\lambda_{IC}$) allow to improve the eventual unbalance between the partial losses during the training process. These weights can be user-specified or automatically tuned. In this work, the network architecture (e.g., hidden layers, neurons per layer) and hyperparameters (e.g., learning rate, loss weights) were selected manually. While automated methods exist, their implementation is beyond the scope of this study. 

\section{Experimental Procedure}

Applying PINNs to address both forward and inverse problems in ODE-based dynamical systems requires a sequence of well-defined steps. In the context of PBPK-iPINN, these steps are summarized in Algorithm 1 and described in detail in Section 4.1.

\begin{figure}[H]
\centering
\includegraphics[width=12.0cm]{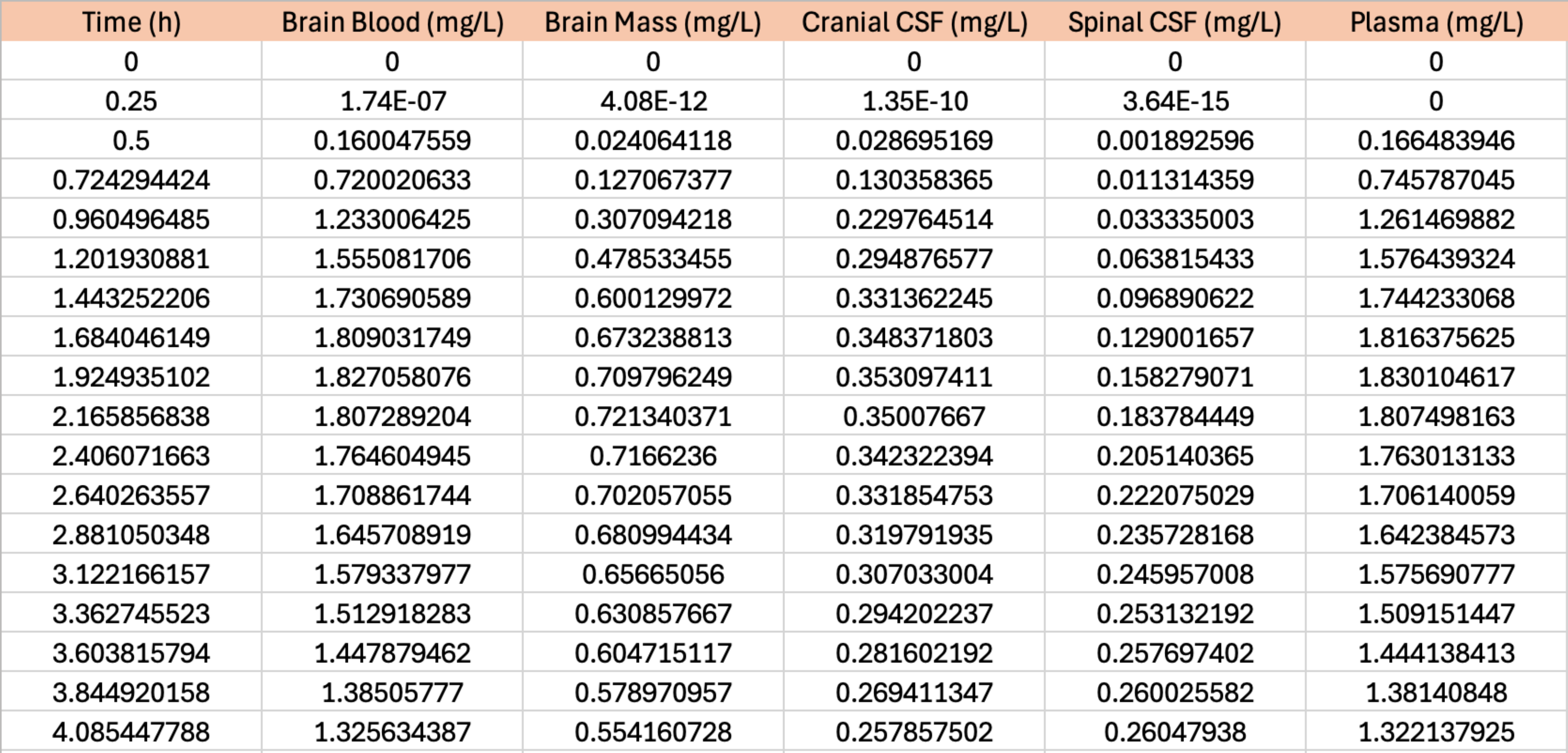}
\begin{tiny}
\caption{Input file showing first few rows of the concentration time data of brain blood, brain mass, carnial CSF, spinal CSF, and plasma.}\label{bdata}
\end{tiny}
\end{figure}

\begin{algorithm}[H]
\caption{: PBPK-iPINN Methodology}
\label{alg:pinn_parameter_estimation}
\begin{algorithmic}[1]
\STATE \textbf{Step 1:  Import data}
\STATE \quad Read timepoints and concentration data from the input file
\STATE
\STATE \textbf{Step 2: Define exogenous input function}
\STATE \quad Define a linear interpolation function for a continuous plasma profile
\STATE
\STATE \textbf{Step 3: Specify the parameters to be estimated, initialize computational domain, and  enforce initial conditions for state variables}
\STATE \quad Parameter = \textcolor{blue}{dde.Variable}
(initial value) 
\STATE \quad Create a \textcolor{blue}{TimeDomain class}
\STATE \quad Define a function to return points inside a subdomain.
\STATE \quad Specify initial conditions using \textcolor{blue}{dde.icbc.IC}.
\STATE
\STATE \textbf{Step 4: Formulate the system of ODEs}
\STATE \quad Define a function to return residuals of ODEs
\STATE \quad Enforce parameter transformations to guarantee positivity if needed.
\STATE
\STATE \textbf{Step 5: Assign training data points and  assemble data module}
\STATE \quad Use \textcolor{blue}{dde.icbc.PointSetBC} to assign training data
\STATE \quad Assemble data module using  \textcolor{blue}{dde.data.PDE} 
\STATE
\STATE \textbf{Step 6: Construct neural network architecture}
\STATE \quad Define the number of neurons and layers
\STATE \quad Specify the activation function(e.g., ReLu, Tanh)
\STATE \quad Specify the weight and bias initializers (Glorot uniform, Glorot normal)
\STATE
\STATE \textbf{Step 7: Setup, compile, train, and predict}
\STATE \quad Combine data and the network architecture :  \textcolor{blue}{dde.Model(data, net)} 
\STATE \quad Define parameters to be estimated as external trainable variables
\STATE \quad Assign optimization algorithm (adam), learning rate, and loss weights
\STATE \qquad $
\mathcal{L} = \lambda_{Data} \mathcal{L}_{\text{data}} + \lambda_{ODE} \mathcal{L}_{\text{ODE}} + \lambda_{IC}  \mathcal{L}_{\text{IC}}
$
\STATE \qquad $\theta^* = \arg\min_{\theta} \mathcal{L}_{\text{total}}(\theta)$
\STATE \quad Specify the numebr of iterations and train the model:  \textcolor{blue}{model.train(iterations=k)}
\STATE \quad Predict solution for set of time points: \textcolor{blue}{model.predict(time)}
\end{algorithmic}
\end{algorithm}

\subsection{A Step-by-Step Guide to Implementing a Four-Compartment Brain PINN in DeepXDE}  Numerous software libraries, including DeepXDE, SimNet, PyDEns, NeuroDiffEq, NeuralPDE, SciANN, ADCME, GPyTorch, and Neural Tangents, are specifically designed for physics-informed machine learning. We selected DeepXDE \cite{lu2021deepxde} to implement the Physics-Informed Neural Network (PINN) methodology for our four-compartment brain model. This library provided advanced features and robust network architectures that were essential for achieving productive and trustworthy results. Its compatibility with standard tools like the Anaconda Python distribution allows for simple import into a Jupyter notebook. A detailed description of each component of the Algorithm 1 is given in as  follows: 

\begin{enumerate}[align=left, leftmargin=*, labelwidth=!, itemindent=0pt]
    \item \textbf{Import Data:} \textit{ Data are generated using  Simcyp™ PBPK Simulator. First, a virtual population of 100 healthy volunteers was generated. Each individual in the population received a single oral dose of 10 mg abemaciclib drug. Subsequently, the mean drug concentrations across the population were recorded for each compartment of the brain. Finally, the mean plasma concentrations were also recorded. Concentration data are recorded up to 48 hours with 200 time points. A portion of the data file is shown in figure (\ref{bdata}) due to space limitation.  }
 
\item \textbf{Define exogenous input function:} \textit{The plasma concentration data are provided as input to the system via the function \( g_1( t;\theta) \), as defined in Equation (\ref{exis}). Since the plasma time - concentration data are available at discrete time points, linear interpolation is applied to construct a continuous approximation. As a result, \( g(\theta, t) \) becomes a continuous function of time. To this end we implemented following function (PlasmaInterp) to our program where "observe t" and "plasma" are user provided two discrete vectors. It is also worth noting that in practice plasma concentration data are also measured with errors.}

\begin{lstlisting}[style=mypython]
def PlasmaInterp(t):
    spline = sp.interpolate.Rbf(observe_t, plasma,
                             function="linear", smooth=0, epsilon=0)
    return spline(t[:,0:])
\end{lstlisting}

\item \textbf{Specify the parameters to be estimated, initialize computational domain, and enforce initial conditions for state variables:} \textit{For this analysis, and for the sake of simplicity, we estimate only six parameters. The selected parameters are \( V_{\mathrm{bb}} \), \( V_{\mathrm{bm}} \), \( V_{\mathrm{ccsf}} \), \( V_{\mathrm{scsf}} \), \( fu_{\mathrm{bb}} \), and \( \lambda_{\mathrm{ccsf}} \). However, any number or combination of parameters can be chosen depending on the modeling objectives. This selection is specified using the built in function called \texttt{dde.Variable} as follows where an educational guess of the initial parameters can be provided.}

\begin{lstlisting}[style=mypython]
Vbb = dde.Variable(1.0)
Vbm = dde.Variable(1.0)
Vccsf = dde.Variable(1.0)
Vscsf = dde.Variable(1.0)
fubb = dde.Variable(1.0)
lamccsf = dde.Variable(1.0)
\end{lstlisting}

\noindent \textit{The computational domain is defined using the built in function  \\ \texttt{dde.geometry.TimeDomain}}.

\begin{lstlisting}[style=mypython]
geom = dde.geometry.TimeDomain(0, maxtime)
\end{lstlisting}

\noindent  \textit{ Initial conditions are called as follows where, \texttt{x0[0]} is the initial value  of the first state variable.}
    
\begin{lstlisting}[style=mypython]
ic1 = dde.icbc.IC(geom, lambda X: x0[0], boundary, component=0)
ic2 = dde.icbc.IC(geom, lambda X: x0[1], boundary, component=1)
ic3 = dde.icbc.IC(geom, lambda X: x0[2], boundary, component=2)
ic4 = dde.icbc.IC(geom, lambda X: x0[3], boundary, component=3)
\end{lstlisting}

\item \textbf{Formulate the system of ODEs:} \textit{The system of ordinary differential equations (ODEs) defined through the equations (\ref{eqn1}), (\ref{eqn2}), (\ref{eqn3}) and (\ref{eqn4})  is defined as a set of residual functions. The gradients of the state variables are specified in the following way: }
    
\begin{lstlisting}[style=mypython]
dCbb_x    = dde.grad.jacobian(y, x, i=0)
dCbm_x    = dde.grad.jacobian(y, x, i=1)
dCccsf_x  = dde.grad.jacobian(y, x, i=2)
dCscsf_x  = dde.grad.jacobian(y, x, i=3)
\end{lstlisting}

\noindent \textit{In order to maintain the positivity of the parameters and to avoid unrealistic parameter values during the training process we apply sigmoid transformation to each parameters to be trained, which also helps to maintain the mass balance principle described by the differential equations.}

\begin{lstlisting}[style=mypython]
Vbb     = min1 + (max1 - min1) * torch.sigmoid(Vbb_p)
Vbm     = min2 + (max2 - min2) * torch.sigmoid(Vbm_p)
Vccsf   = min3 + (max3 - min3) * torch.sigmoid(Vccsf_p)
Vscsf   = min4 + (max4 - min4) * torch.sigmoid(Vscsf_p)
fubb    = min5 + (max5 - min5) * torch.sigmoid(fubb_p)
lamccsf = min6 + (max6 - min6) * torch.sigmoid(lamccsf_p)
\end{lstlisting}

\item \textbf{Assign training data points and assemble data module:} \textit{The concentration data of the four compartments are are incorporated into the model training process in the following way: }

\begin{lstlisting}[style=mypython]
Obs_Cbb   = dde.icbc.PointSetBC(Obs_t,
            Obs_Data['Cbb'].values.reshape(-1,1), component=0)
Obs_Cbm   = dde.icbc.PointSetBC(Obs_t,
            Obs_Data['Cbm'].values.reshape(-1,1), component=1)
Obs_Cccsf = dde.icbc.PointSetBC(Obs_t,
            Obs_Data['Cccsf'].values.reshape(-1,1), component=2)
Obs_Cscsf = dde.icbc.PointSetBC(Obs_t,
            Obs_Data['Cscsf'].values.reshape(-1,1), component=3)
\end{lstlisting}

\noindent \textit{We incorporate geometry, system of ODEs, initial conditions, concentration data, number of domain points, number of boundary points, and additional training points via the data module. For example, the data module of the four-compartmental model can be specified as follows:}
\begin{lstlisting}[style=mypython]
data = dde.data.PDE(geom, B4_system, [ic1, ic2, ic3, ic4,
       Obs_Cbb, Obs_Cbm, Obs_Cccsf, Obs_Cscsf],
       num_domain = 1, num_boundary = 2,
       anchors = all_anchors,
       auxiliary_var_function = PlasmaInterp)
\end{lstlisting}

\item \textbf{Construct neural network architecture:} \textit{The number of input vectors (a), number of neuron per layer (b), number of layers (c), number of output vectors (d), type of the activation function (tanh) and the, the type of a weight initialization (Glorot normal) strategy can be assigns as follows:}

\begin{lstlisting}[style=mypython]
net = dde.nn.FNN([a] + [b] * c + [d], "tanh", "Glorot normal")
\end{lstlisting}

\item \textbf{Setup, compile, train, and predict:}
\textit{ We set up the model by combining the data and the network. Then we list the parameters to be trained and pass it to model compile where we use the Adam's optimization algorithm with a specified learning rate(k). Then we train the model for a given number of epochs. Once the model is trained we predict the output for any given set of time points.}

\begin{lstlisting}[style=mypython]
loss_weights = np.concatenate([
    np.ones(4) * 1.0,
    np.ones(4) * 2.0,
    np.ones(4) * 3.0
])
loss_weights = loss_weights.tolist()
\end{lstlisting}

\begin{lstlisting}[style=mypython]
model = dde.Model(data, net)
paramaters = [Vbb_p, Vbm_p, Vccsf_p, Vscsf_p, fubb_p, lamccsf_p]
model.compile("adam", lr = k, external_trainable_variables = paramaters)
model.train(iterations = max_iter)
model.predict(time_new)
\end{lstlisting}

\end{enumerate}

\section{Numerical Results} 
This section presents our main results and their validation. In \textbf{Example 1}, we identify the optimal configuration including the activation function, number of layers and neurons, weight and bias initializers (Glorot uniform and Glorot normal), optimizer, and learning rate that minimizes the total loss (objective function). In \textbf{Example 2}, we solve the inverse problem, estimating the parameters \( V_{\mathrm{bb}} \), \( V_{\mathrm{bm}} \), \( V_{\mathrm{ccsf}} \), \( V_{\mathrm{scsf}} \), \( fu_{\mathrm{bb}} \), and \( \lambda_{\mathrm{ccsf}} \) and presenting the corresponding predicted concentration profiles. We further present the results of the practical identifiability analysis. Finally, in \textbf{Example 3}, we validate these results using two additional numerical and statistical techniques.

\vspace{0.5cm}

\textbf{Example 1.}
In this example, we test four activation functions ReLU, Sigmoid, Tanh, and Sine using varying numbers of layers and neurons per layer. We evaluate prediction accuracy by recording the total loss (as defined in Section~\ref{sec:3}). For weight and bias initialization, we compare the ``Glorot Uniform'' and ``Glorot Normal'' methods, ultimately selecting Glorot Uniform because it yields the lowest loss value. We use the Adam optimizer with a learning rate of $0.0001$, running each simulation for 10000 iterations. During training, we record both the loss values and the training time to determine the optimal hyperparameter combinations.

As shown in Table~(\ref{tab:4}), the Tanh and Sigmoid activation functions generally perform best for the four-compartment brain model. Figure~(\ref{allfig}) further confirms their  accuracy compared to Sin and ReLU by overlaying observed data with the predicted concentration profiles from each activation function. The training times in Table~(\ref{tab:4}) align with the expected trend that computational cost increases with the number of layers and neurons. However, we observed several occasions where the order of the error (total loss) is of the magnitude of $10^{-2}$ after 10000 iterations. Therefore, we further investigate the best network architecture in Example 2. 

\begin{figure}[H]
\centering
\includegraphics[width=12cm]{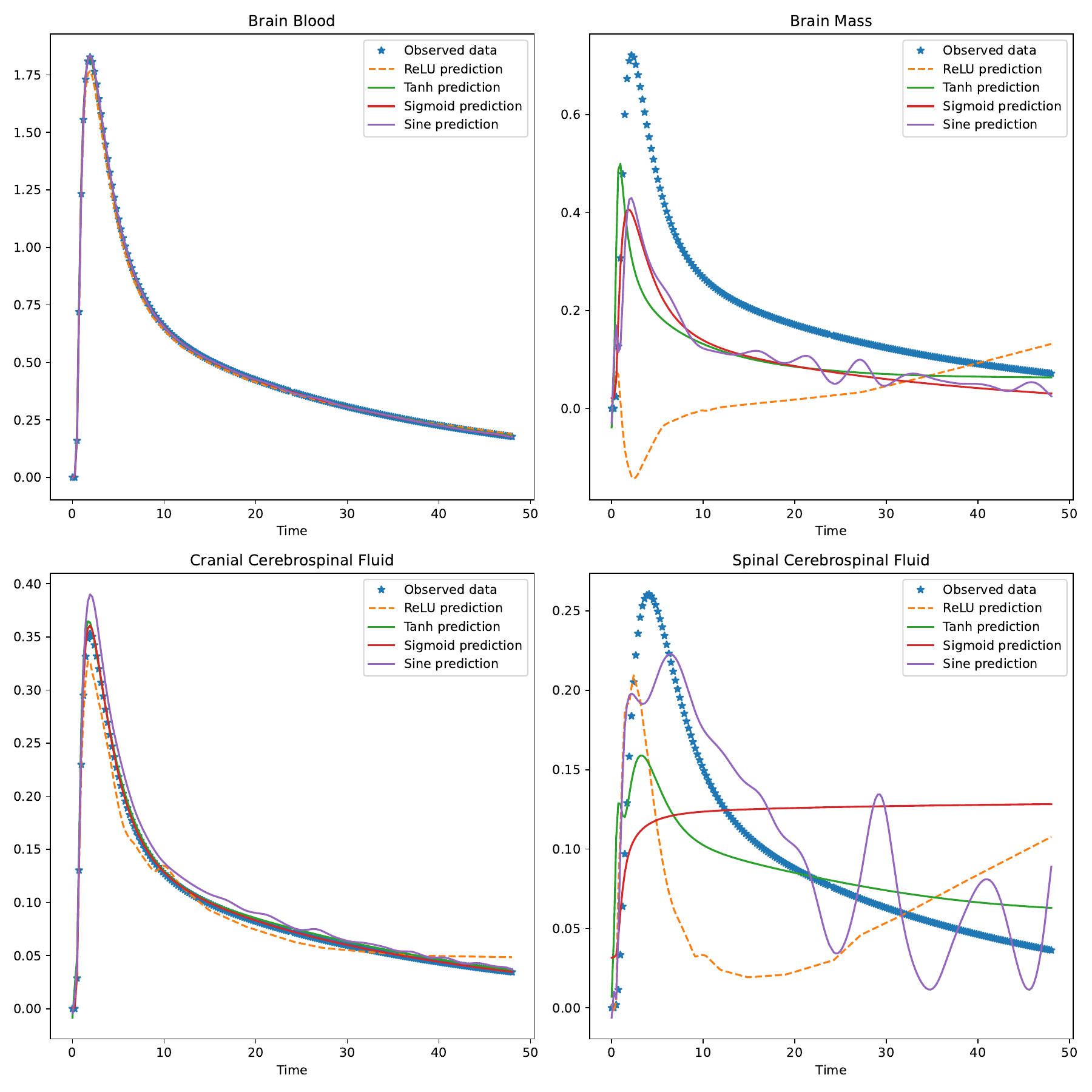}
\begin{tiny}
\caption{Estimated drug concentration profiles of brain compartments for different activation functions.}\label{allfig}
\end{tiny}
\end{figure}

\begin{table}[htbp]
\centering
\caption{The total loss and the training time for different number of layers (L), activation functions (AF) and neurons (N).}\label{tab:4}
\footnotesize
\begin{tabular}{c l c c c c}
\hline
\textbf{L} & \textbf{AF} & \multicolumn{4}{c}{\textbf{Neurons(N)}} \\
\hline
 & & \textbf{N=9} & \textbf{N=18} & \textbf{N=27} & \textbf{N=50} \\
\hline
\multirow{4}{*}{1} & ReLU    & 1.56e+04 (161) & 1.56e+04 (165) & 1.56e+04 (171) & 6.69e+02 (264) \\
 & Tanh    & 1.37e-01 (199) &  4.58e-02 (210) &  4.99e-02 (222) & 3.97e-02 (307) \\
 & Sigmoid & 1.77e-01 (171) &  5.67e-02 (173) &  2.96e-02 (187) &  4.05e-02 (327)\\
 & Sine     & 9.73e+02 (216) &  5.49e+02 (236) &  5.07e+02 (237) &  5.41e+02 (368)\\
\hline
\multirow{4}{*}{2} & ReLU    &  6.43e+02 (194) &  7.47e+01 (217) &  1.55e+04 (232) &  3.95e+01 (374)\\
 & Tanh    &  6.67e-02 (284)&  6.82e-02 (257) &  3.44e-02 (284) &  4.57e-02 (377)\\
 & Sigmoid &  1.03e-01 (196) &  2.44e-02 (231) &  5.83e-02 (240) &  4.24e-02 (438)\\
 & Sine     &  7.21e+00 (315) &  1.52e+00 (342) &  8.50e-01 (293) &  3.95e+00 (656)\\
\hline
\multirow{4}{*}{6} & ReLU    &  1.53e+00 (319) &  2.78e+00 (372) &  1.19e-01 (405) &  1.50e-01(659)\\
 & Tanh    &  4.68e-02 (435)& 5.39e-02 (515) &  3.91e-02 (615) & 4.20e-02 (1017)\\
 & Sigmoid &  1.40e+04 (396) &  8.59e-02 (451) &  5.30e-02 (527) &  6.15e-02 (1200)\\
 & Sine     &  5.69e-02 (571) &  3.16e-02 (573) & 4.24e-02 (724) &  4.57e-02 (1393)\\
\hline
\end{tabular}
\end{table}

\textbf{Example 2}. In this example, we consider estimating parameters up to six parameters:  \( V_{\mathrm{bb}} \), \( V_{\mathrm{bm}} \), \( V_{\mathrm{ccsf}} \), \( V_{\mathrm{scsf}} \), \( fu_{\mathrm{bb}} \), and \( \lambda_{\mathrm{ccsf}} \) though any other parameter set could also be estimated if they are structurally identifiable as shown in the case 5  in Table (\ref{tab:3}). However, selecting a large number of parameters comes at the cost of increased \textit{computational expense}.
First, we present the results of the practical identification analysis.  The correlation matrix derived from the Fisher Information Matrix as shown in Table (\ref{tab:5}) indicates that most parameter pairs exhibit low to moderate correlations. In particular, the absolute values of the majority of correlation coefficients are below 0.26, suggesting weak linear dependence between parameters and good practical identifiability. The highest correlation is observed between \( fu_{\mathrm{bb}} \), and \( \lambda_{\mathrm{ccsf}} \)  (0.4646), which indicates a moderate positive association but is not sufficiently strong to suggest severe practical non-identifiability. The high correlation between \( fu_{\mathrm{bb}} \), and \( \lambda_{\mathrm{ccsf}} \) compared to the other parameters is able to explain the high absolute errors of \( fu_{\mathrm{bb}} \), and \( \lambda_{\mathrm{ccsf}} \) as shown in the Table (\ref{tab:6}). Overall, the absence of strong correlations (e.g., $|r| > 0.8$ ) implies that the selected parameters can be estimated with acceptable reliability from the available data.

\begin{table}[h!]
\centering
\caption{Correlation matrix of estimated parameters}\label{tab:5}
\begin{tabular}{lcccccc}
\hline
 & Vbb & Vbm & Vccsf & Vscsf & fubb & lamccsf \\
\hline
Vbb      & 1        &          &           &           &           &           \\
Vbm      & -0.2542  & 1        &           &           &           &           \\
Vccsf    & 0.05475  & -0.1292  & 1         &           &           &           \\
Vscsf    & 0.05873  & -0.02423 & -0.151    & 1         &           &           \\
fubb     & -0.1456  & 0.1764   & -0.009497 & -0.1639   & 1         &           \\
lamccsf  & -0.02003 & -0.08028 & 0.1277    & -0.01126  & 0.4646    & 1         \\
\hline
\end{tabular}
\label{tab:correlation_matrix}
\end{table}

From Table (\ref{tab:4}) we found that there are several occasions that the error (total loss) is of the magnitude of $10^{-2}$ after 10000 iterations. Among these cases we selected 
"tanh" activation function with Glorot normal initializer for initial weight and bias generator. The selected neural network archtechchor contains 6 layers and 50 neurons per layers since this selection shows the closer approximation to the reference parameter values as summarized in the Table (\ref{tab:4a}) after 10000 iterations. 

\begin{table}[h!]
\centering
\caption{Selection of the best neural network architecture after 10,000 iterations}
\label{tab:4a}
\begin{tabular}{lccc}
\hline
 & Case 1: (6, 50, tanh) & Case 2: (6, 18, sine) & Case 3: (1, 27, sigmoid)  \\
\hline
Vbb      & 0.042984428107738494  & 0.013211678266525268  & 0.022292330414056777 \\
Vbm      & 0.704248666763305664  & 0.617162728309631347 & 0.730870246887207031 \\
Vccsf    & 0.093138557821512222  & 0.057842567563056945  & 0.060063321143388748 \\
Vscsf    & 0.015588552691042423  & 0.010563577413558959  & 0.011213994637131690 \\
fubb     & 0.091953142285346984  & 0.063309334218502044  & 0.072771822810173034 \\
lamccsf  & 0.019866755947470664 & 0.011166308075189590  & 0.014584582298994064\\
\hline
\end{tabular}
\label{tab:correlation_matrix}
\end{table}

To further improve the solution we then trained the network for 5 million iterations, which took 54977.47 seconds to complete. After training the model with adam's optimizer with the learning rate of 0.0001 we further train the model with Limited-memory Broyden-Fletcher-Goldfarb-Shanno optimizer for further smoothness and the best model is found at the 5000018 iteration with total loss of $1.23e-05$. As evidence in Table (\ref{tab:6}) it can be seen that the parameters \( V_{\mathrm{bb}} \), \( V_{\mathrm{bm}} \), \( V_{\mathrm{ccsf}} \) and \( V_{\mathrm{scsf}} \) were trained with very high accuracy. The estimated parameter values of  \( fu_{\mathrm{bb}} \), and \( \lambda_{\mathrm{ccsf}}\) are physiologically acceptable and accurate enough. However, based on the level of accuracy that we expect the model can be further trained.

\begin{table}[htbp]
\centering
\def\arraystretch{1.0}
\caption{Comparison of reference parameter values, the model estimated parameter values, absolute errors, and 95\% confidence intervals (CIs).} 
\label{tab:6}
\tabcolsep=6pt
\footnotesize
\begin{tabular}{|>{\RaggedRight}p{0.15\linewidth}|c|c|c|c|}
\hline
\textbf{Parameter} & \textbf{Reference Value} & \textbf{PINN Value} & \textbf{Absolute Error} & \textbf{95\% CI} \\
\hline
$V_{bb}$ & 0.064952435 & 0.064952425 & $9.34 \times 10^{-9}$ & $[0.064816,0.065088]$ \\
\hline
$V_{bm}$ & 1.104115461 & 1.104115366 & $9.41 \times 10^{-8}$ & $[1.104095,1.104135]$ \\
\hline
$V_{ccsf}$ & 0.103984624 & 0.103984609 & $1.48 \times 10^{-8}$ & $[0.103723,0.104246]$ \\
\hline
$V_{scsf}$ & 0.025996156 & 0.025996146 & $9.28 \times 10^{-9}$ & $[0.025931,0.026062]$ \\
\hline
$fu_{bb}$ & 0.125 & 0.128736447 & $3.74 \times 10^{-3}$ & $[0.111186,0.146286]$ \\
\hline
$\lambda_{ccsf}$ & 0.026 & 0.021295592 & $4.70 \times 10^{-3}$ & $[0.012936, 0.029656]$ \\
\hline
\end{tabular}
\end{table}

In terms of improving the results through the loss weights, we manually, set the weights (instead of letting network to learn the loss weights) for $\lambda_{IC}$, $\lambda_{ODE}$ and $\lambda_{Data}$ by observing the predictions and comparing the loss curves. In this way, we found that $\lambda_{IC}=(1, 1, 1, 1)$, $\lambda_{ODE} = (2, 2, 2, 2)$, and $\lambda_{Data}=(3, 3, 3, 3)$ yielded the best results in our experiments.

Figure (\ref{pco}) illustrates the evolution and stabilization of the estimated parameter values over training epochs. The training was initialized with all parameter values set to zero, and a Sigmoid transformation was applied to each parameter to ensure positivity throughout the optimization process.

\begin{figure}[H]
\centering
\includegraphics[width=14.0cm]{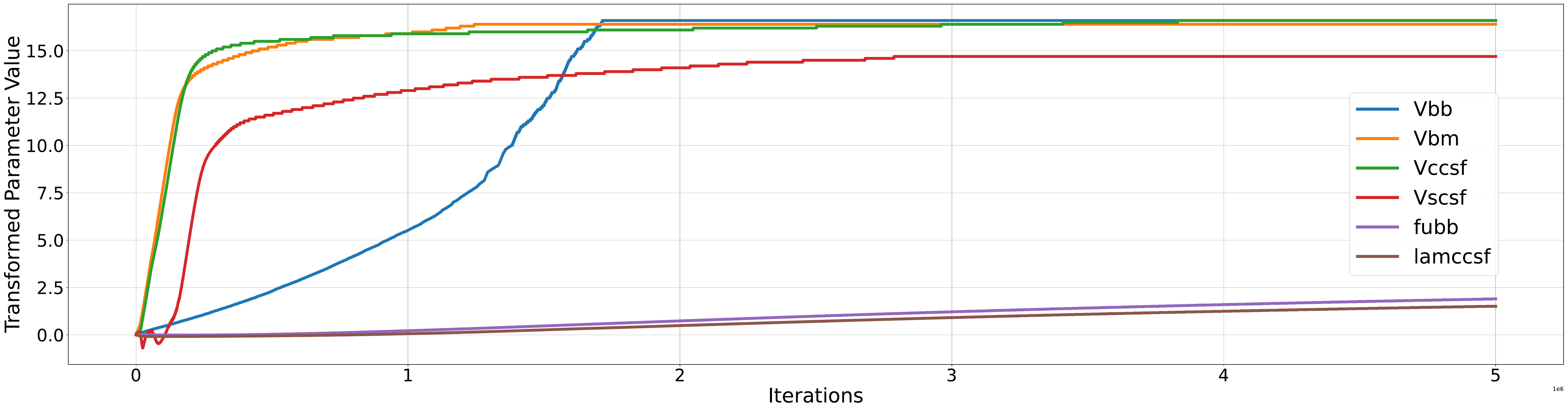}
\begin{tiny}
\caption{Stabilization of estimated parameters with sigmoid transformation over epochs.}\label{pco}
\end{tiny}
\end{figure}

Figure (\ref{dp}) presents the ODE loss, initial condition loss, data loss, and total loss over training epochs for the optimal weight values of $\lambda_{Data}, \lambda_{ODE},$ and $\lambda_{IC}$. The initial condition loss remains consistently smaller than both the ODE loss and data loss, as the initial values of the ODE state variables are identical to those of the observed data. As shown in Figure (\ref{dp}), the loss function continues to improve even after 5 million iterations. Nevertheless, by this point the model achieves the desired level of accuracy in terms of both parameter estimation and the individual fits of the drug concentration profiles across compartments. 

\begin{figure}[H]
\centering
\includegraphics[width=14.0cm]{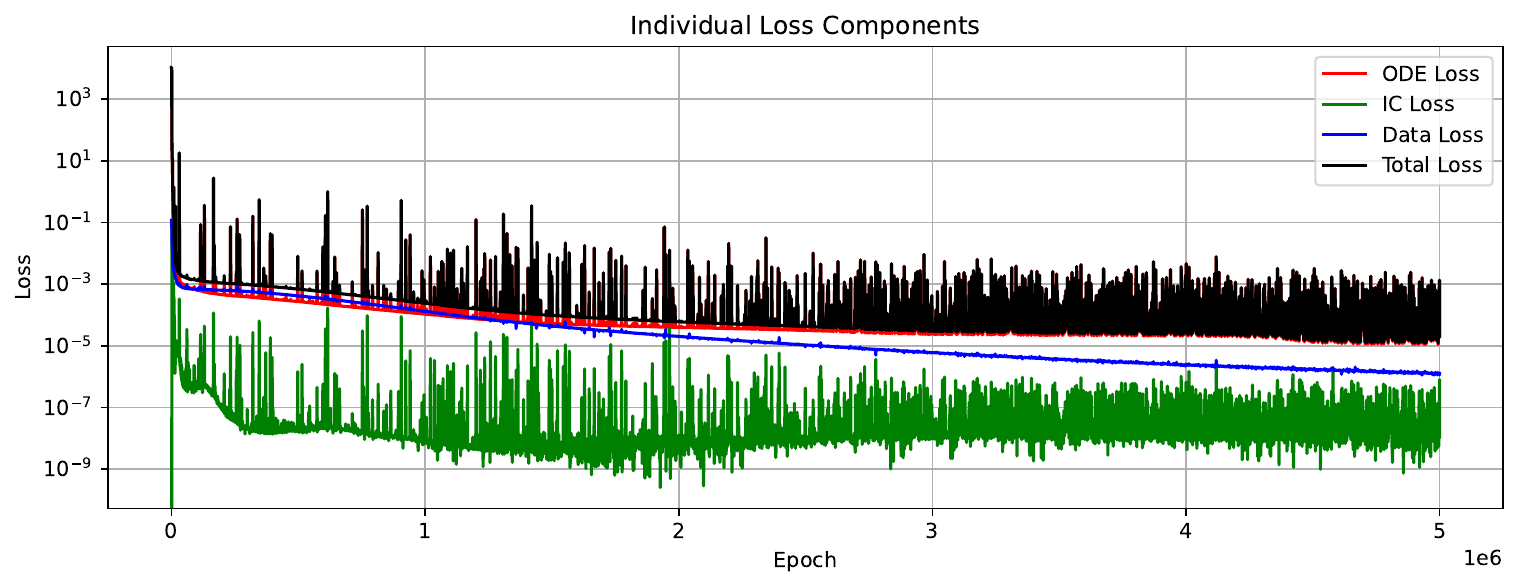}
\begin{tiny}
\caption{Minimization of weighted log(Loss) components over epoch.}\label{dp}
\end{tiny}
\end{figure}

Figure (\ref{figall}) reports the forward solution of the ODE system using the estimated parameters. The resulting drug concentration profiles in brain blood ($C_{bb}$), brain mass ($C_{bm}$), cranial CSF ($C_{ccsf}$), and spinal CSF ($C_{scsf}$) exhibit close agreement with the observed data, consistent with the behavior of the data loss component.

\begin{figure}[H]
\centering
\includegraphics[width=15.0cm]{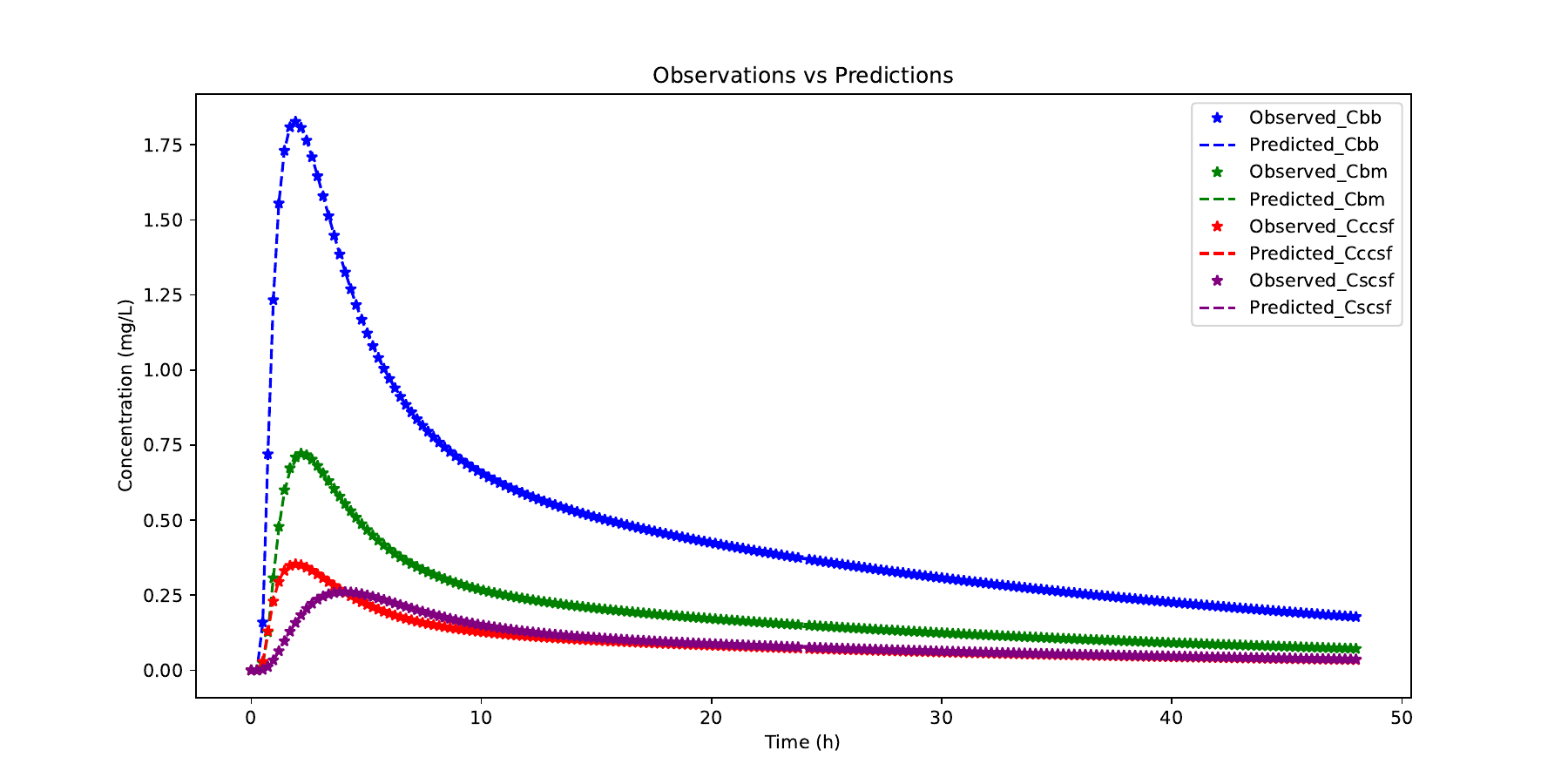}
\begin{tiny}
\caption{Drug concentration vs time for observed and predicted concentrations.}\label{figall}
\end{tiny}
\end{figure}

\textbf{Example 3}. 
In this example, we compare the prediction accuracy of our PBPK-iPINN approach by comparing it with traditional statistical and numerical methods. First, the Stochastic Approximation Expectation-Maximization (SAEM) algorithm \cite{chan2011use}, implemented in MONOLIX, was used to estimate model parameters. The forward problem was then solved using explicit Runge-Kutta (RK) schemes. Second, the Differential Evolution (DE) algorithm \cite{ardia2016package}, implemented in R, was used for parameter estimation. The forward problem was subsequently solved using the Livermore Solver for Ordinary Differential Equations with Adaptive step-size control, which automatically switches between a non-stiff method (Adams) and a stiff method (Backward Differentiation Formula, BDF) depending on problem stiffness. Physiologically meaningful parameter bounds were applied across all methods to constrain the parameter space and accelerate convergence. For a fair comparison, identical parameter bounds were defined for the PBPK-iPINN, SAEM, and DE algorithms. Absolute errors (for example, $| \text{PINN}_{\text{err}} | = | \text{Reference Value} - \text{PINN} |$), reported in Table (\ref{tab:8}), were calculated relative to the reference parameter values and rounded to five decimal places.

The results of the SAEM algorithm (implemented in MONOLIX) were obtained after 2000 iterations with 1000 Monte Carlo importance sampling steps. The simulation took 3421 seconds to complete, with the stopping criterion set to a tolerance of $1 \times 10^{-8}$. Similarly, the results of the DE algorithm (implemented in R) were obtained after 2000 iterations, with the relative tolerance set to $1 \times 10^{-8}$, and required 4083 seconds to complete. All simulations were performed on a MacBook Air (M1, 2020) with 8 GB RAM. The PINN simulation was completed as described in Example~2. Although PINN is computationally expensive, PBPK-iPINN demonstrates promising potential as an alternative to traditional numerical and statistical methods for addressing complex inverse problems, particularly those involving highly stiff systems.

\begin{table}[H] 
\centering
\def\arraystretch{1.0}
\caption{Comparison of estimated parameter values with absolute errors.} \label{tab:8}
\tabcolsep=4pt
\begin{tabular}[c]{|c|c|c|c|c|c|c|}
\hline
Parameter & PINN & SEAM & DE & $|{\rm PINN_{err}}|$ & $|SEAM_{err}|$ & $|{DE_{err}}|$ \\
\hline
$V_{bb}$        & 0.06495 & 0.06500 & 0.06497 & 0.00000 & 0.00005 & 0.00002 \\
\hline
$V_{bm}$        & 1.10412 & 1.10000 & 1.10347 & 0.00000 & 0.00412 & 0.00065 \\
\hline
$V_{ccsf}$      & 0.10398 & 0.10000 & 0.10455 & 0.00000 & 0.00398 & 0.00057 \\
\hline
$V_{scsf}$      & 0.02599 & 0.02600 & 0.02600 & 0.00000 & 0.00001 & 0.00001 \\
\hline
$fu_{bb}$       & 0.12874 & 0.12000 & 0.12500 & 0.00374 & 0.00500 & 0.00000 \\
\hline
$\lambda_{ccsf}$ & 0.02600 & 0.02700 & 0.02600 & 0.00000 & 0.00100 & 0.00000 \\
\hline
\end{tabular}
\end{table}

In Figure (\ref{fig:combined}), we overlay the drug concentration profiles obtained from SAEM, DE, and PINN with the reference solution to visually assess the approximation capability of PINN. The results demonstrate that the PINN solution approximates the true solution equally well or better than the other methods.

\begin{figure}[H]
    \centering
    \begin{subfigure}[b]{0.40\textwidth}
        \centering
        \includegraphics[width=\linewidth]{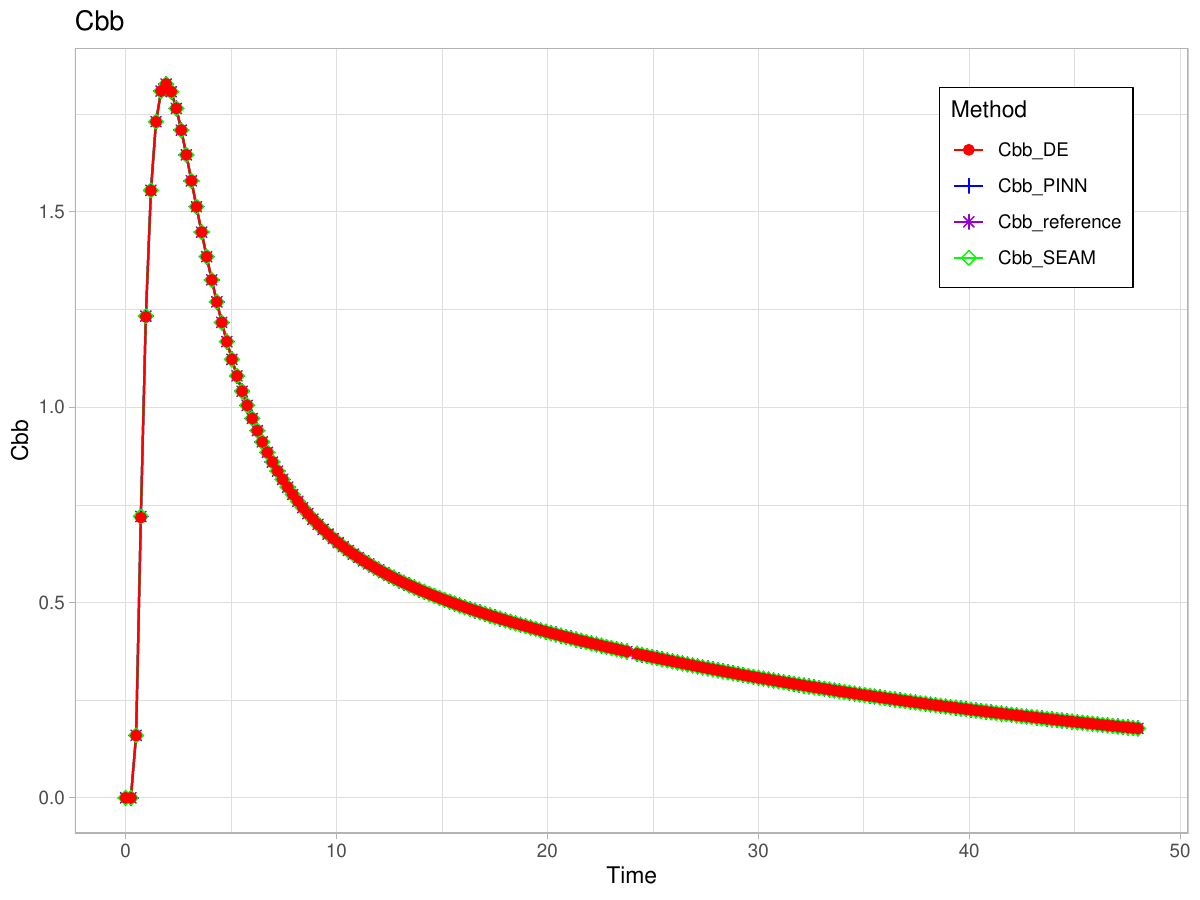}
        \caption{Cbb profiles.}
        \label{fig:sub1a}
    \end{subfigure}
    \hfill
    \begin{subfigure}[b]{0.40\textwidth}
        \centering
        \includegraphics[width=\linewidth]{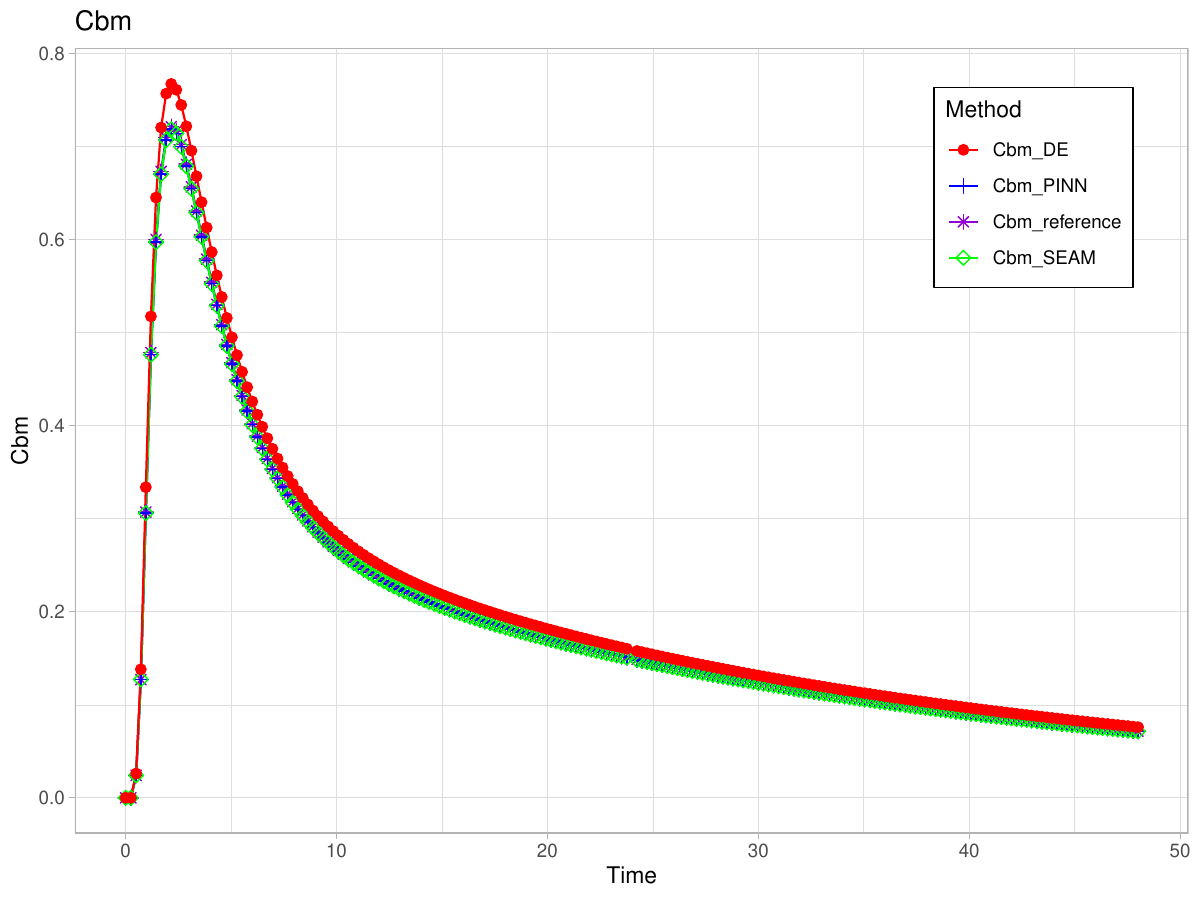}
        \caption{Cbm profiles.}
        \label{fig:sub2a}
    \end{subfigure}
    
    \vspace{0.5cm} 
    
    \begin{subfigure}[b]{0.40\textwidth}
        \centering
        \includegraphics[width=\linewidth]{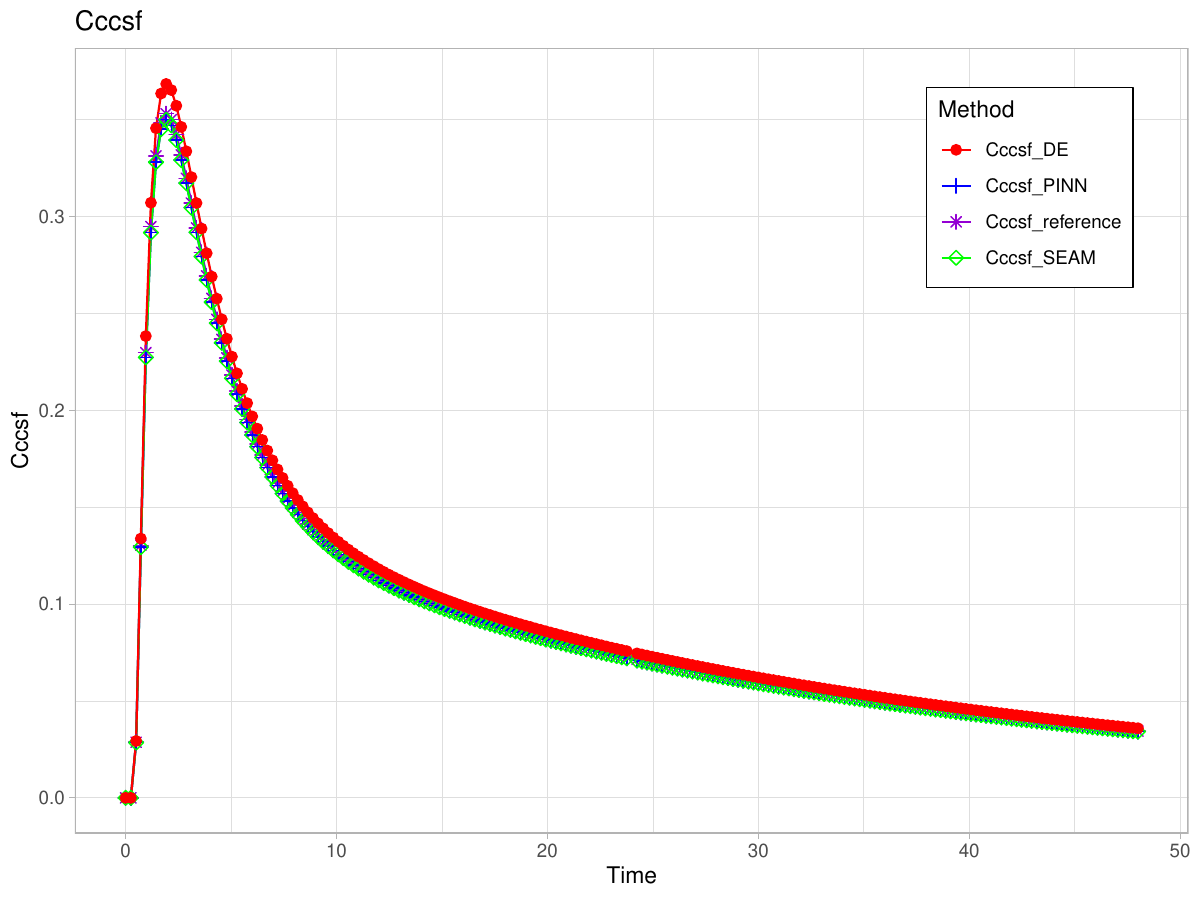}
        \caption{Cccsf profiles.}
        \label{fig:sub3a}
    \end{subfigure}
    \hfill
    \begin{subfigure}[b]{0.40\textwidth}
        \centering
        \includegraphics[width=\linewidth]{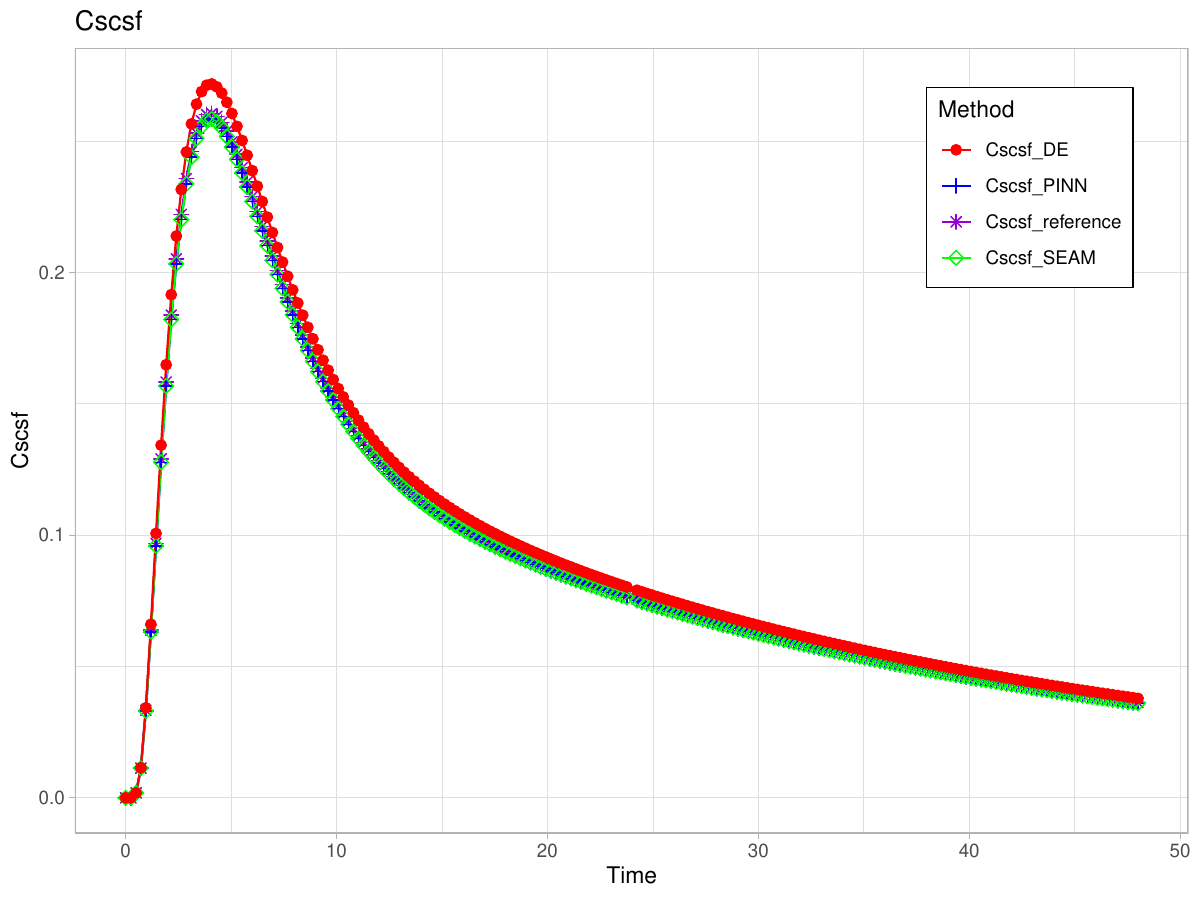}
        \caption{Cscsf profiles.}
        \label{fig:sub4a}
    \end{subfigure}
    
    \caption{Concentration - time profiles of the 4-compartments obtained from PINN, SEAM, DE, algorithms and using the reference parameter values.}
    \label{fig:combined}
\end{figure}

\section{Conclusion}
To the best of our knowledge, research on inverse physics-informed neural networks (iPINNs) applied to physiologically based pharmacokinetic (PBPK) modeling is limited, and no study has applied iPINNs to the 4-compartment, permeability-limited PBPK brain model used in this work for estimating drug- and patient-specific parameters and predicting drug concentration profiles. 

The PBPK-iPINN framework provides a robust, efficient, and accurate alternative for estimating parameters and drug concentration profiles. The implications of this approach are twofold. First, it offers a powerful tool for analyzing a wide range of compounds within the established 4-compartment brain model, streamlining drug development and research. Second, the framework is adaptable to any multi-compartment PBPK model, existing or future, where in vitro experimentation or traditional computational approaches (e.g., numerical or statistical algorithms) are inadequate. This study helps determine AUC, Tmax, and Cmax through accurate estimation of parameters and concentration profiles, which could assist drug developers and healthcare providers in developing more effective drugs and optimizing the use of current therapies to treat brain cancer.

\textbf{Limitations and Future Work}

Experimentally determining drug-specific or patient-specific parameters can be a complex and time-consuming process due to challenges in sampling and limitations of analytical technologies (e.g., in vitro experiments). Furthermore, existing numerical and statistical algorithms often fail to converge when the system of ODEs becomes highly stiff. Since this study recommends PBPK-iPINN as an alternative to conventional numerical and statistical methods, we plan to apply the same methodology described in this paper to metabolic flux analysis (MFA) in future work. MFA consists of a large number of differential equations (more than 60 ODEs) and more than 80 parameters to be estimated, where we have already observed that traditional numerical and statistical methods fail to converge.

\section*{Conflict of Interest}
The authors declare no conflict of interest.

\section*{Author Contributions}
C.W. wrote the manuscript; C.W. designed the research; C.W., K.W., and P.R. performed the PINN simulations; C.W. and N.H. analyzed the data and conducted the parameter identifiability analysis.

\section*{Acknowledgments}
No funding was received for this research.

\section*{Data Availability}
The Python codes used with the DeepXDE library and the drug concentration time profiles generated through Simcyp simulator are available upon request.

\end{document}